\documentclass{article}

\usepackage{graphicx,amsmath,amsfonts,amssymb,bm,url,epsfig,epsf,color,fullpage,mathbbol,fmtcount,algorithmic,algorithm,semtrans,caption,subcaption,multirow,amsthm}

\usepackage{enumitem}
\usepackage[title]{appendix}
\usepackage[utf8]{inputenc} 
\usepackage[T1]{fontenc}    
\usepackage{url}            
\usepackage{booktabs}       
\usepackage{amsfonts}       
\usepackage{nicefrac}       
\usepackage{microtype}      
\usepackage{hhline}

\usepackage[T1]{fontenc}
\usepackage[utf8]{inputenc}
\usepackage{graphicx}
\usepackage{subcaption}
\usepackage{color}
\usepackage{amsmath}
\usepackage{bbm}
\usepackage{tcolorbox}
\usepackage{lipsum}  
\usepackage{tcolorbox}
\usepackage{amsfonts,amssymb}
\usepackage{mathtools}
\usepackage{commath}
\usepackage{relsize}
\usepackage{comment,color,soul}
\usepackage{amsfonts}
\usepackage{url}
\usepackage{lipsum}
\usepackage{nicefrac}
\usepackage{float}
\usepackage{titling}

\usepackage[hidelinks,backref=page]{hyperref}
\definecolor{darkred}{RGB}{150,0,0}
\definecolor{darkgreen}{RGB}{0,150,0}
\definecolor{darkblue}{RGB}{0,0,150}
\hypersetup{colorlinks=true, linkcolor=darkred, citecolor=darkblue, urlcolor=darkblue}

\date{}

\author{\\ Farzan Farnia, Asuman Ozdaglar\\
\{farnia,asuman\}@mit.edu\\\\
Massachusetts Institute of Technology\\}

\usepackage[round]{natbib}

\newtheorem{thm}{Theorem}
\newtheorem*{thm*}{Theorem}
\newtheorem{remark}{Remark}
\newtheorem*{remark*}{Remark}
\newtheorem{lemma}{Lemma}
\newtheorem{cor}{Corollary}
\newtheorem{cor*}{Corollary*}
\newtheorem{mydef}{Definition}

\newtheorem{assumption}{Assumption}

\newcommand{\myeqref}[1]{{(\ref{#1})}}

\def\rmI{{\mathbf{I}}}

\def\rmZ{{\mathbf{Z}}}

\DeclareMathAlphabet{\mathsfit}{\encodingdefault}{\sfdefault}{m}{sl}
\SetMathAlphabet{\mathsfit}{bold}{\encodingdefault}{\sfdefault}{bx}{n}

\newcommand{\btheta}{\boldsymbol{\theta}}
\newcommand{\gda}{\text{\rm GDA}}
\newcommand{\sgda}{\text{\rm SGDA}}
\newcommand{\gdmax}{\text{\rm GDmax}}
\newcommand{\sgdmax}{\text{\rm SGDmax}}
\newcommand{\ppm}{\text{\rm PPM}}
\newcommand{\sppm}{\text{\rm SPPM}}
\newcommand{\ppmax}{\text{\rm PPmax}}
\newcommand{\sppmax}{\text{\rm SPPmax}}
\newcommand{\gen}{\text{\rm gen}}
\newcommand{\bw}{\mathbf{w}}
\newcommand{\bu}{\mathbf{u}}
\newcommand{\bz}{\mathbf{z}}
\newcommand{\bZ}{\mathbf{Z}}
\newcommand{\bv}{\mathbf{v}}
\newcommand{\bbE}{\mathbb{E}}

\allowdisplaybreaks
\newcommand\numberthis{\addtocounter{equation}{1}\tag{\theequation}}

\title{Train simultaneously, generalize better:\\ Stability of gradient-based minimax learners}

\begin{document}
\maketitle

\begin{abstract}
The success of minimax learning problems of generative adversarial networks (GANs) has been observed to depend on the minimax optimization algorithm used for their training. This dependence is commonly attributed to the convergence speed and robustness properties of the underlying optimization algorithm. In this paper, we show that the optimization algorithm also plays a key role in the \emph{generalization performance} of the trained minimax model. To this end, we analyze the generalization properties of standard gradient descent ascent (GDA) and proximal point method (PPM) algorithms through the lens of algorithmic stability under both convex concave and non-convex non-concave minimax settings. 
While the GDA algorithm is not guaranteed to have a vanishing excess risk in convex concave problems, we show the PPM algorithm enjoys a bounded excess risk in the same setup. For non-convex non-concave problems, we compare the generalization performance of stochastic GDA and GDmax algorithms where the latter fully solves the maximization subproblem at every iteration. Our generalization analysis suggests the superiority of GDA provided that the minimization and maximization subproblems are solved simultaneously with similar learning rates. We discuss several numerical results indicating the role of optimization algorithms in the generalization of the learned minimax models.

\end{abstract}

\section{Introduction}
Minimax learning frameworks including generative adversarial networks (GANs) \citep{goodfellow2014generative} and adversarial training \citep{madry2017towards} have recently achieved great success over a wide array of learning tasks. In these frameworks, the learning problem is modeled as a zero-sum game between a ''min" and ''max" player that is solved by a minimax optimization algorithm. The minimax optimization problem of these learning frameworks is typically formulated using deep neural networks, which greatly complicates the theoretical and numerical analysis of the optimization problem. Current studies in the machine learning literature focus on fundamental understanding of general minimax problems with emphasis on convergence speed and optimality.  

The primary focus of optimization-related studies of minimax learning problems has been on the convergence and robustness properties of minimax optimization algorithms. Several recently proposed algorithms have been shown to achieve faster convergence rates and more robust behavior around local solutions. However, training speed and robustness are not the only factors required for the success of a minimax optimization algorithm in a learning task. In this work, our goal is to show that the \emph{generalization performance} of the learned minimax model is another key property that is influenced by the underlying optimization algorithm. To this end, we present theoretical and numerical results demonstrating that: 
\begin{description}[wide, labelwidth=!,labelindent=0pt]\item\emph{Different minimax optimization algorithms can learn models with different generalization properties.}
\end{description}

In order to analyze the generalization behavior of minimax optimization algorithms, we use the algorithmic stability framework as defined by \cite{bousquet2002stability} for general learning problems and applied by \cite{hardt2016train} for analyzing stochastic gradient descent. 
Our extension of \citep{bousquet2002stability}'s stability approach to minimax settings allows us to analyze and compare the generalization properties of standard gradient descent ascent (GDA) and proximal point method (PPM) algorithms. Furthermore, we compare the generalization performance between the following two types of minimax optimization algorithms: 1) simultaneous update algorithms such as GDA where the minimization and maximization subproblems are simultaneously solved, and 2) non-simultaneous update algorithms such as GDmax where the maximization variable is fully optimized at every iteration. 

In our generalization analysis, we consider both the traditional convex concave and general non-convex non-concave classes of minimax optimization problems. For convex concave minimax problems, our bounds indicate a similar generalization performance for simultaneous and non-simultaneous update methods. Specifically, we show for strongly-convex strongly-concave minimax problems all the mentioned algorithms have a bounded generalization risk on the order of $O(1/n)$ with $n$ denoting the number of training samples. However, in general convex concave problems we show that the GDA algorithm with a constant learning rate is not guaranteed to have a bounded generalization risk. On the other hand, we prove that proximal point methods still achieve a controlled generalization risk, resulting in a vanishing $O(\sqrt{1/n})$ excess risk with respect to the best minimax learner with the optimal performance on the underlying distribution. 

For more general minimax problems, our results indicate that models trained by simultaneous and non-simultaneous update algorithms can achieve different generalization performances. Specifically, we consider the class of non-convex strongly-concave problems where we establish stability-based generalization bounds for both stochastic GDA and GDmax algorithms. Our generalization bounds indicate that the stochastic GDA learner is expected to generalize better provided that the min and max players are trained simultaneously with similar learning rates. In addition, we show a generalization bound for the stochastic GDA algorithm in general non-convex non-concave problems, which further supports the simultaneous optimization of the two min and max players in general minimax settings. Our results indicate that simultaneous training of the two players not only can provide faster training, but also can learn a model with better generalization performance. Our generalization analysis, therefore, revisits the notion of \emph{implicit competitive regularization} introduced by \cite{schafer2019implicit} for simultaneous gradient methods in training GANs. 


Finally, we discuss the results of our numerical experiments and compare the generalization performance of GDA and PPM algorithms in convex concave settings and single-step and multi-step gradient-based methods in non-convex non-concave GAN problems. Our numerical results also suggest that in general non-convex non-concave problems the models learned by simultaneous optimization algorithms can generalize better than the models learned by non-simultaneous optimization methods. We can summarize the main contributions of this paper as follows:
\begin{itemize}[wide, labelwidth=!,labelindent=0pt]
    \item Extending the algorithmic stability framework for analyzing generalization in minimax settings,
    \item Analyzing the generalization properties of minimax models learned by GDA and PPM algorithms in convex concave problems,
    \item Studying the generalization of stochastic GDA and GDmax learners in non-convex non-concave problems,
    \item Providing numerical results on the role of optimization algorithms in the generalization performance of learned minimax models. 
\end{itemize}

\section{Related Work}
\textbf{Generalization in GANs:} Several related papers have studied the generalization properties of GANs. \cite{arora2017generalization} study the generalization behavior of GANs' learned models and prove a uniform convergence generalization bound in terms of the number of the discriminator's parameters. \cite{wu2019generalization} connect the algorithmic stability notion to differential privacy in GANs and numerically analyze the generalization behavior of GANs. References \citep{zhang2017discrimination,bai2018approximability} show uniform convergence bounds for GANs by analyzing the Rademacher complexity of the players. \cite{feizi2020understanding} provide a uniform convergence bound for the W2GAN problem. Unlike the mentioned related papers, our work provides algorithm-dependent generalization bounds by analyzing the stability of gradient-based optimization algorithms. Also, the related works \citep{arora2017gans,thanh2019improving} conduct empirical studies of generalization in GANs using birthday paradox-based and gradient penalty-based approaches, respectively.\vspace{0.2cm}

\textbf{Generalization in adversarial training:} Understanding generalization in the context of adversarial training has recently received great attention. \cite{schmidt2018adversarially} show that in a simplified Gaussian setting generalization in adversarial training requires more training samples than standard non-adversarial learning. \cite{farnia2018generalizable,yin2019rademacher,khim2018adversarial,wei2019improved,attias2019improved} prove uniform convergence generalization bounds for adversarial training schemes through Pac-Bayes \citep{mcallester1999some,neyshabur2017pac}, Rademacher analysis, margin-based, and VC analysis approaches. \cite{zhai2019adversarially} study the value of unlabeled samples in obtaining a  better generalization performance in adversarial training. We note that unlike our work the generalization analyses in the mentioned papers prove uniform convergence results. In another related work, \cite{rice2020overfitting} empirically study the generalization performance of adversarially-trained models and suggest that the generalization behavior can significantly change during training.\vspace{0.2cm} 

\textbf{Stability-based generalization analysis:} Algorithmic stability and its connections to the generalization properties of learning algorithms have been studied in several related works. \cite{shalev2010learnability}
discuss learning problems where learnability is feasible considering algorithmic stability, while it is infeasible with uniform convergence. 
\cite{hardt2016train} bound the generalization risk of the stochastic gradient descent learner by analyzing its algorithmic stability.
\cite{feldman2018generalization,feldman2019high,bousquet2020sharper}
provide sharper stability-based generalization bounds for standard learning problems. While the above works focus on standard learning problems with a single learner, we use algorithmic stability to analyze generalization in minimax settings with two players.\vspace{0.2cm}   

\textbf{Connections between generalization and optimization in deep learning:} The connections between generalization and optimization in deep learning have been studied in several related works. Analyzing the double descent phenomenon \citep{belkin2019reconciling,nakkiran2019deep,mei2019generalization}, the effect of 
overparameterization on generalization \citep{li2018learning,allen2019learning,arora2019fine,cao2019generalization,wei2019regularization,bietti2019inductive,allen2019can,ongie2019function,ji2019polylogarithmic,bai2019beyond}, and the sharpness of local minima \citep{keskar2016large,dinh2017sharp,neyshabur2017exploring} have been performed in the literature to understand the implicit regularization of gradient methods in deep learning \citep{neyshabur2014search,zhang2016understanding,ma2018implicit,lyu2019gradient,chatterjee2020coherent}. \cite{schafer2019implicit} extend the notion of implicit regularization to simultaneous gradient methods in GAN settings and discuss an optimization-based perspective to this regularization mechanism. However, we focus on the generalization aspect of the implicit regularization mechanism. Also, \cite{nagarajan2019uniform} suggest that uniform convergence bounds may be unable to explain generalization in supervised deep learning.\vspace{0.2cm}

\textbf{Analyzing convergence and stability of minimax optimization algorithms:} A large body of related papers \citep{heusel2017gans,sanjabi2018convergence, lin2019gradient,schafer2019competitive,fiez2019convergence,nouiehed2019solving,hsieh2019finding,du2019linear,wang2019solving,mazumdar2019finding,thekumparampil2019efficient,farnia2020gans,mazumdar2020gradient,zhang2020newton} study convergence properties of first-order and second-order minimax optimization algorithms. Also, the related works
\citep{daskalakis2017training,daskalakis2018limit,gidel2018variational,liang2019interaction,mokhtari2020unified} analyze the convergence behavior of optimistic methods and extra gradient (EG) methods as approximations of the proximal point method. We also note that we use the algorithmic stability notion as defined by \cite{bousquet2002stability}, which is different from the local and global stability properties of GDA methods around optimal solutions studied in the related papers 
\citep{mescheder2017numerics,nagarajan2017gradient,mescheder2018training,feizi2020understanding}.

\section{Preliminaries}
In this paper, we focus on two standard families of minimax optimization algorithms: Gradient Descent Ascent (GDA) and Proximal Point Method (PPM). To review the update rules of these algorithms, consider the following minimax optimization problem for minimax objective $f(\bw,\btheta)$ and feasible sets $\mathcal{W},\Theta$:  
\begin{equation}\label{Eq: General Minimax Problem}
    \min_{\bw\in\mathcal{W}}\; \max_{\btheta\in \Theta}\; f(\bw,\btheta).
\end{equation}
Then, for stepsize values $\alpha_w,\alpha_\theta$, the followings are the GDA's and GDmax's update rules:
\begin{align}
    G_{\gda}(\begin{bmatrix}
    \bw \\ 
    \btheta
    \end{bmatrix}) := \begin{bmatrix}
    \bw - \alpha_w\nabla_{\bw}f(\bw,\btheta) \\ 
    \btheta + \alpha_{\theta}\nabla_{\btheta}f(\bw,\btheta)
    \end{bmatrix},\quad G_{\gdmax}(\begin{bmatrix}
    \bw \\ 
    \btheta
    \end{bmatrix}) := \begin{bmatrix}
    \bw - \alpha_w\nabla_{\bw}f(\bw,\btheta) \\ 
    {\arg\!\max}_{\widetilde{\btheta}\in\Theta}f(\bw,\widetilde{\btheta})
    \end{bmatrix}
\end{align}
In the above, $ {\arg\!\max}_{\theta\in\Theta}f(\bw,\btheta)$ is the optimal maximizer for $\bw$. 
Also, given stepsize parameter 
$\eta$ the update rule of PPM is as follows:
\begin{align}
    G_{\ppm}(\begin{bmatrix}
    \bw \\ 
    \btheta
    \end{bmatrix}) := 
    \underset{\widetilde{\bw}\in\mathcal{W}}{\arg\!\min}\:\underset{\widetilde{\btheta}\in\Theta}{\arg\!\max}\: \bigl\{f(\widetilde{\bw},\widetilde{\btheta}) +\frac{1}{2\eta}\Vert \widetilde{\bw} - \bw \Vert^2_2-\frac{1}{2\eta}\Vert \widetilde{\btheta} - \btheta \Vert^2_2\bigr\}, 
\end{align}
In the Appendix, we also consider and analyze the PPmax algorithm that is a proximal point method fully solving the maximization subproblem at every iteration.
Throughout the paper, we commonly use the following assumptions on the Lipschitzness and smoothness of the minimax objective.
\begin{assumption}
$f(\bw,\btheta)$ is jointly $L$-Lipschitz in $(\bw,\btheta)$ and $L_w$-Lipschitz in $\bw$ over $\mathcal{W}\times\Theta$, i.e., for every $\bw,\bw'\in\mathcal{W},\, \btheta,\btheta'\in \Theta$ we have
\begin{align*}
      &\big\vert f(\bw,\btheta) - f(\bw',\btheta')\big\vert  \le L\sqrt{\Vert\bw-\bw' \Vert^2_2+\Vert\btheta-\btheta' \Vert^2_2},\\ 
     &\big\vert f(\bw,\btheta) - f(\bw',\btheta) \big\vert \le L_w\Vert\bw-\bw' \Vert_2. \numberthis
\end{align*}
\end{assumption}
\begin{assumption}
$f(\bw,\btheta)$ is continuously differentiable and $\ell$-smooth on $\mathcal{W}\times\Theta$, i.e., $\bigl[ \nabla_\bw f(\bw,\btheta), \nabla_{\btheta} f(\bw,\btheta) \bigr]$ is $\ell$-Lipschitz on $\mathcal{W}\times\Theta$ and for every $\bw,\bw'\in\mathcal{W},\, \btheta,\btheta'\in \Theta$ we have
\begin{align*}
      {\bigl\Vert \nabla_\bw f(\bw,\btheta) - \nabla_\bw f(\bw',\btheta') \bigr\Vert^2_2 + \bigl\Vert \nabla_{\btheta} f(\bw,\btheta)- \nabla_{\btheta} f(\bw',\btheta') \bigr\Vert^2_2}  \le \ell^2 \,\bigl(\Vert\bw-\bw' \Vert^2_2+\Vert\btheta-\btheta' \Vert^2_2\bigr).\numberthis
\end{align*}
\end{assumption}
We focus on several classes of minimax optimization problems based on the convexity properties of the objective function. Note that a differentiable function $g(\mathbf{u})$ is called convex in $\mathbf{u}$ if it satisfies the following inequality for every $\bu_1,\bu_2$:
\begin{equation}
    g(\bu_2)\ge g(\bu_1) + \nabla g(\bu_1)^{\top}(\bu_2-\bu_1).
\end{equation}
Furthermore, $g$ is called $\mu$-strongly-convex if for every $\bu_1,\bu_2$ it satisfies 
\begin{equation}
     g(\bu_2)\ge g(\bu_1) + \nabla g(\bu_1)^{\top}(\bu_2-\bu_1) + \frac{\mu}{2}\Vert\bu_2-\bu_1 \Vert^2_2 .
\end{equation}
Also, $g$ is called concave and $\mu$-strongly-concave if $-g$ is convex and $\mu$-strongly-convex, respectively.
\begin{mydef}
Consider convex feasible sets $\mathcal{W},\Theta$ in  minimax problem \myeqref{Eq: General Minimax Problem}. Then,
\begin{itemize}[wide, labelwidth=!,labelindent=0pt]
    \item The problem is called convex concave if $f(\cdot,\btheta)$ and $f(\bw,\cdot)$ are respectively convex and concave functions for every $\bw,\btheta$.
    \item The  problem is called $\mu$-strongly-convex strongly-concave if $f(\cdot,\btheta)$ and $f(\bw,\cdot)$ are respectively $\mu$-strongly-convex and $\mu$-strongly-concave functions for every $\bw,\btheta$.
    \item The problem is called non-convex $\mu$-strongly-concave if $f(\bw,\cdot)$ is  $\mu$-strongly-concave for every $\bw$.
\end{itemize}
\end{mydef}

\section{Stability-based Generalization Analysis in Minimax Settings}

Consider the following optimization problem for a minimax learning task:
\begin{equation}
    \min_{\bw\in\mathcal{W}}\; \max_{\btheta\in\Theta}\; R(\bw,\btheta):=\bbE_{\mathbf{Z}\sim P_Z}\bigl[f(\bw,\btheta;\mathbf{Z}) \bigr]
\end{equation}
The above minimax objective represents a cost function $f(\bw,\btheta;\mathbf{Z})$ for minimization and maximization variables $\bw,\btheta$ and data variable $\rmZ$ that is averaged under the underlying distribution $P_{\rmZ}$. We call the objective function $R(\bw,\btheta)$ the true minimax risk. We also define $R(\bw)$ as the worst-case minimax risk over the maximization variable $\btheta$:
\begin{equation}
  R(\bw):= \max_{\btheta\in \Theta}\;R(\bw,\btheta)
\end{equation}
In the context of GANs, the worst-case risk $R(\bw)$ represents a divergence measure between the learned and true distributions, and in the context of adversarial training it represents the learner's risk under adversarial perturbations. Since the learner does not have access to the underlying distribution $P_\bZ$, we estimate the minimax objective using the empirical samples in dataset $S=(\bz_1,\ldots,\bz_n)$ which are drawn according to $P_{\rmZ}$. We define the empirical minimax risk as:
\begin{equation}\label{Eq: Minimax Learning Empirical}
     R_S(\bw,\btheta):= \frac{1}{n}\sum_{i=1}^n\, f(\bw,\btheta;\bz_i). 
\end{equation}
Then, the worst-case empirical risk over the maximization variable $\btheta$ is defined as
\begin{equation}
  R_S(\bw):= \max_{\btheta \in\Theta}\;R_S(\bw,\btheta).
\end{equation}
We define the minimax generalization risk $\epsilon_{\gen}(\bw)$ of minimization variable $\bw$ as the difference between the worst-case true and empirical risks:
\begin{equation}
\epsilon_{\gen}(\bw):= R(\bw) - R_S(\bw).
\end{equation}
The above generalization score measures the difference of empirical and true worst-case minimax risks. For a randomized algorithm $A$ which outputs random outcome $A(S)=(A_w(S),A_{\theta}(S))$ for dataset $S$ we define $A$'s expected generalization risk as
\begin{equation}
    \epsilon_{\text{\rm gen}}(A) := \mathbb{E}_{S,A}\bigl[R(A_w(S)) - R_S(A_w(S))\bigr].
\end{equation}
\begin{mydef}
A randomized minimax optimization algorithm $A$ is called $\epsilon$-uniformly stable in minimization if for every two datasets $S,S'\in \mathcal{Z}^n$ which differ in only one sample, for every $\mathbf{z}\in\mathcal{Z},\btheta\in\Theta$ we have
\begin{equation}
     \mathbb{E}_A\bigl[\,f(A_w(S),\btheta;\bz)-f(A_w(S'),\btheta;\bz) \,\bigr]\le \epsilon.
\end{equation}
\end{mydef}
Considering the above definition, we show the following theorem that connects the definition of uniform stability to the generalization risk of the learned minimax model.
\begin{thm}\label{Thm: Stability and Generalization}
Assume minimax learner $A$ is $\epsilon$-uniformly stable in minimization. Then, $A$'s expected generalization risk is bounded as 
\begin{equation}
\epsilon_{gen}(A)\le \epsilon.
\end{equation}
\end{thm}
\begin{proof}
We defer the proof to the Appendix.
\end{proof}
In the following sections, we apply the above result to analyze generalization for convex concave and non-convex non-concave minimax learning problems.

\section{Generalization Analysis for Convex Concave Minimax Problems}
Analyzing convergence rates for convex concave minimax problems is well-explored in the optimization literature. Here, we use the algorithmic stability framework to bound the expected generalization risk in convex concave minimax learning problems. 
We start by analyzing the generalization risk in strongly-convex strongly-concave problems. The following theorem applies the stability framework to bound the expected generalization risk  under this scenario.  
\begin{thm}\label{Thm: Gen Strongly-Convex Strongly-Concave}
Let minimax learning objective $f(\cdot,\cdot;\bz)$ be $\mu$-strongly-convex strongly-concave and satisfy Assumption 2 for every $\bz$. Assume that Assumption 1 holds for convex-concave $\widetilde{f}(\bw,\btheta;\bz):=f(\bw,\btheta;\bz)+\frac{\mu}{2}(\Vert \btheta\Vert^2_2-\Vert \bw\Vert^2_2)$ and every $\bz$. 
Then, full-batch and stochastic GDA and GDmax algorithms with stepsize $\alpha_w=\alpha_{\theta}\le \frac{\mu}{\ell^2}$  will satisfy the following  bounds over $T$ iterations:
    \begin{equation}
        \epsilon_{\gen}(\gda)\le \frac{2LL_w}{(\mu-\frac{\alpha_w\ell^2}{2})n},\quad \epsilon_{\gen}(\gdmax)\le \frac{2L^2_w}{\mu n}.
    \end{equation}
\end{thm}
\begin{proof}
We defer the proof to the Appendix. In the Appendix, we also prove similar bounds for full-batch and stochastic proximal point methods.
\end{proof}
Note that regarding Assumption 1 in the above theorem, we suppose the assumption holds for the deregularized $\widetilde{f}$, because a strongly-convex strongly-concave objective cannot be Lipschitz over an unbounded feasible set. We still note that the theorem's bounds will hold for the original $f$ if in Assumption 1  we define $f$'s Lipschitz constants over bounded feasible sets $\mathcal{W},\Theta$.  

Given sufficiently small stepsizes for GDA, Theorem \ref{Thm: Gen Strongly-Convex Strongly-Concave} suggests a similar generalization performance between GDA and GDmax which are different by a factor of $L/L_w$. 
For general convex concave problems, it is well-known in the minimax optimization literature that the GDA algorithm can diverge from an optimal saddle point solution. As we show in the following remark, the generalization bound suggested by the stability framework will also grow exponentially with the iteration count in this scenario. 
\begin{remark}\label{Remark: convex concave constant stepsize}
Consider a convex concave minimax objective $f(\cdot,\cdot;\bz)$ satisfying Assumptions 1 and 2. Given constant stepsizes $\alpha_w=\alpha_{\theta}=\alpha$, the GDA's generalization risk over $T$ iterations will be bounded as: 
\begin{equation}
   \epsilon_{\gen}(GDA)\le O\bigl(\frac{\alpha LL_w(1+\alpha^2\ell^2)^{T/2}}{n}\bigr). 
\end{equation}
In particular, the bound's exponential dependence on $T$ is tight for the GDA's generalization risk in the special case of $f(\bw,\btheta;\bz)=\bw^{\top}(\bz-\btheta)$.
\end{remark}
\begin{proof}
We defer the proof to the Appendix.
\end{proof}
On the other hand, proximal point methods have been shown to resolve the convergence issues of GDA methods in convex concave problems \citep{mokhtari2019convergence,mokhtari2020unified}. Here, we also show that these algorithms enjoy a generalization risk growing at most linearly with $T$.
\begin{thm}\label{Thm: Gen Convex-Concave}
Consider a convex-concave minimax learning objective $f(\cdot,\cdot;\bz)$ satisfying Assumptions 1 and 2 for every $\bz$. Then, full-batch and stochastic PPM with parameter $\eta$ will satisfy the following bound over $T$ iterations:
\begin{equation}
    \epsilon_{\gen}(\ppm)\le \frac{2\eta LL_w T}{n}.
\end{equation}
\end{thm}
\begin{proof}
We defer the proof to the Appendix. In the Appendix, we also show a similar bound for the PPmax algorithm. 
\end{proof}
The above generalization bound allows us to analyze the true worst-case minimax risk of PPM learners in convex concave problems. To this end, we decompose the true worst-case risk into the sum of the stability and empirical worst-case risks and optimize the sum of these two error components' upper-bounds. Note that Theorem \ref{Thm: Gen Convex-Concave} bounds the generalization risk of PPM in terms of stepsize parameter $\eta$ and number of iterations $T$. Therefore, we only need to bound the iteration complexity of PPM's convergence to an $\epsilon$-approximate saddle point. To do this, we show the following theorem that extends \cite{mokhtari2019convergence}'s result for PPM to stochastic PPM.
\begin{thm}\label{Thm: PPM stochastic convergence}
Given a differentiable minimax objective $f(\bw,\btheta;\bz)$ the average iterate updates $\bar{\bw}^{(T)}:=\frac{1}{T}\sum_{t=1}^T \bw^{(t)},\, \bar{\btheta}^{(T)}:=\frac{1}{T}\sum_{t=1}^T \btheta^{(t)}$ of  stochastic PPM (SPPM) with setpsize parameter $\eta$ will satisfy the following for a saddle point $[\bw^*_S,\btheta^*_S]$ of the empirical risk under dataset $S$:
\begin{align*}
   \mathbb{E}_A\bigl[ R_S({\bar{\bw}}^{(T)}) \bigr]- R_S(\bw^*_S)
   \le\frac{\big\Vert [\bw^{(0)},\btheta^{(0)}]-[\bw^*_S,\btheta^*_S]\big\Vert_2^2}{2\eta T}.\numberthis
    \end{align*}
\end{thm}
\begin{proof}
We defer the proof to the Appendix. In the  Appendix, we also prove a similar result for stochastic PPmax.
\end{proof}
The above convergence result suggests that the expected empirical worst-case risk of applying $T$ iterations of stochastic PPM will be at most $O(1/\eta T)$. In addition, Theorem \ref{Thm: Gen Convex-Concave} shows that using that number of iterations the generalization risk will be bounded by $O({\eta T}/{n})$. Minimizing the sum of these two error components, the following corollary bounds the excess risk suffered by the PPM algorithm. 
\begin{cor}\label{Cor: Convex Concave}
Consider a convex concave minimax objective and a proximal point method optimizer with constant parameter $\eta$. Given that $\Vert\bw^{(0)}-\bw^*\Vert^2+\Vert\btheta^{(0)}-\btheta^*\Vert^2\le D^2$ holds with probability $1$ for optimal saddle solution $(\bw^*,\btheta^*)$ of the minimax risk, it will take $T_{\ppm}=\sqrt{\frac{nD^2}{2\eta^2LL_w}}$  iterations for the average iterate $\bar{\bw}^{(T)} = \frac{1}{T}\sum_{t=1}^T\bw^{(t)}$ of full-batch and stochastic PPM to have the following bounded excess risk:
\begin{align*}
    \mathbb{E}_{S,A}\bigl[ R(\bar{\bw}^{(T_{\ppm})}) \bigr]- R(\bw^*)\le \sqrt{\frac{2D^2LL_w}{n}}.\numberthis
\end{align*}
\end{cor}
\begin{proof}
We defer the proof to the Appendix. In the Appendix, we prove a similar bound for full-batch and stochastic PPmax as well.
\end{proof}

\section{Generalization Analysis for Non-convex Non-concave Minimax Problems}
In the previous section, we showed that in convex-concave minimax problems simultaneous and non-simultaneous optimization algorithms have similar generalization error bounds which are different by a constant factor ${L}/{L_w}$. However, here we demonstrate that this result does not generalize to general non-convex non-concave problems. We first study the case of non-convex strongly-concave minimax learning problems, where we can analytically characterize the generalization bounds for both stochastic GDA and GDmax algorithms. 
The following theorem states the results of applying the algorithmic stability framework to bound the generalization risk in such minimax problems.
\begin{thm}\label{Thm: Gen non-Convex Strongly-Concave}
Let learning objective $f(\bw,\btheta;\bz)$ be non-convex $\mu$-strongly-concave and satisfy Assumptions 1 and 2. Also, we assume that $f_{\max}(\bw;\bz):=\max_{\btheta\in \Theta}f(\bw,\btheta;\bz)$ is bounded as $0\le f_{\max}(\bw;\bz)\le 1$ for every $\bw,\bz$. Then, defining $\kappa:={\ell}/{\mu}$ we have
\begin{enumerate}[wide,labelwidth=!,labelindent=0pt]
\item The stochastic GDA (SGDA) algorithm with stepsizes $\alpha_{w,t}={c}/{ t},\, \alpha_{\theta,t}={cr^2}/{t}$ for constants $c>0,1\le r\le \kappa$ satisfies the following bound over $T$ iterations: 
    \begin{align}
        \epsilon_{\gen}(\sgda)\le \frac{1+\frac{1}{(r+1)c\ell }}{n} \bigl(12(r+1)cLL_w\bigr)^{\frac{1}{(r+1)c\ell +1}}T^{\frac{(r+1)c\ell }{(r+1)c\ell +1}}.
    \end{align}
    \item The stochastic GDmax (SGDmax) algorithm with stepsize $\alpha_{w,t}={c}/{t}$ for constant $c>0$  satisfies the following bound over $T$ iterations: 
    \begin{align}
        \epsilon_{\gen}(\sgdmax)\le \frac{1+\frac{2}{(\kappa+2)\ell c}}{n} \bigl(2cL_w^2\bigr)^{\frac{2}{(\kappa+2)\ell c+2}}T^{\frac{(\kappa+2)\ell c}{(\kappa+2)\ell c+2}}.
    \end{align}
\end{enumerate}
\end{thm}
\begin{proof}
We defer the proof to the Appendix.
\end{proof}
The above result shows that the generalization risks of stochastic GDA and GDmax change with the number of iterations and training set size as:
\begin{align}
     \epsilon_{\gen}(\sgda)&\approx \mathcal{O}\bigl({T^{\frac{\ell (r+1)c}{\ell (r+1)c+1}}}/{n}\bigr),\nonumber \\
     \epsilon_{\gen}(\sgdmax)&\approx \mathcal{O}\bigl({T^{\frac{\ell (\frac{\kappa}{2}+1)c}{\ell (\frac{\kappa}{2}+1) c+1}}}/{n}\bigr).
\end{align}
Therefore, considering a maximization to minimization stepsize ratio of $r^2 < \kappa^2/4$ will result in a better generalization bound for stochastic GDA compared to stochastic GDmax over a fixed and sufficiently large number of iterations.  

Next, we consider general non-convex non-concave minimax problems and apply the algorithmic stability framework to bound the generalization risk of the stochastic GDA algorithm. Note that the maximized value of a non-strongly-concave function is in general non-smooth. Consequently, the stability framework does not result in a bounded generalization risk for the GDmax algorithm in general non-convex non-concave problems. 
\begin{thm}\label{Thm: Gen non-Convex non-Concave}
Let $0\le f(\cdot,\cdot;\bz)\le 1$ be a bounded non-convex non-concave objective satisfying Assumptions 1 and 2. Then, the SGDA algorithm with stepsizes $\max\{\alpha_{w,t},\alpha_{\theta,t}\}\le {c}/{t}$ for constant $c>0$ satisfies the following bound over $T$ iterations: 
    \begin{equation}
        \epsilon_{\gen}(\sgda)\le \frac{1+\frac{1}{\ell  c}}{n}\bigl(2cLL_w\bigr)^{\frac{1}{\ell c+1}}T^{\frac{\ell c}{\ell c+1}}.
    \end{equation}
\end{thm}
\begin{proof}
We defer the proof to the Appendix.
\end{proof}
Theorem \ref{Thm: Gen non-Convex non-Concave} also shows that the SGDA algorithm with vanishing stepsize values will have a bounded generalization risk of $O(T^{\frac{\ell c}{\ell c+1}}/n)$ over $T$ iterations. On the other hand, the stochastic GDmax algorithm does not enjoy a bounded algorithmic stablility degree in non-convex non-concave problems, since the optimal maximization value behaves non-smoothly in general. 

\section{Numerical Experiments}
Here, we numerically examine the theoretical results of the previous sections. We first focus on a Gaussian setting for analyzing  strongly-convex strongly-concave and convex concave minimax problems. Then, we empirically study generative adversarial networks (GANs) as non-convex non-concave minimax learning tasks. 
\subsection{Convex Concave Minimax Problems}
To analyze our generalization results for convex concave minimax settings, we considered an isotropic Gaussian data vector $\mathbf{Z}\sim\mathcal{N}(\mathbf{0},I_{d\times d})$ with zero mean and identity covariance. In our experiments, we chose $\bZ$'s dimension to be $d=50$. We drew $n=1000$ independent samples from the underlying Gaussian distribution to form a training dataset $S=(\bz_1,\ldots,\bz_n)$. 
For the $\mu$-strongly-convex strongly-concave scenario, we considered the following minimax objective:
\begin{equation}\label{Eq: S.Conv-S.Conc Experiments}
    f_1(\bw,\btheta;\bz)=\bw^{\top}(\bz-\btheta) + \frac{\mu}{2}\bigl( \Vert \bw \Vert^2_2 - \Vert \btheta \Vert^2_2 \bigr).
\end{equation}
In our experiments, we used $\mu=0.1$ and constrained the optimization variables to satisfy the norm bounds $
\Vert \bw\Vert_2, \Vert \btheta\Vert_2\le 100$ which we enforced by projection after every optimization step. Note that for the above minimax objective we have
\begin{equation}\label{Eq: Experiments Conv-Conc gen error}
\epsilon_{\gen}(\bw)= \bw^{\top}(\bbE[\bZ]-{\bbE}_S[\bZ]),  
\end{equation}
where $\bbE[\bZ]=\mathbf{0}$ is the underlying mean and ${\bbE}_S[\bZ]:=\frac{1}{n}\sum_{i=1}^n \bz_i$ is the empirical mean.

To optimize the empirical minimax risk, we applied stochastic GDA with stepsize parameters $\alpha_w=\alpha_{\theta}=0.02$ and stochastic PPM with parameter $\eta=0.02$ each for $T=20,000$ iterations. Figure~\ref{fig:scsc} shows the generalization risk values over the optimization achieved by the stochastic GDA (top) and PPM (bottom) algorithms. As shown in this figure, the absolute value of generalization risk remained bounded during the optimization for both the learning algorithms. In our experiments, we also observed a similar generalization behavior with full-batch GDA and PPM algorithms. We defer the results of those experiments to the supplementary document. Hence, our experimental results support Theorem \ref{Thm: Gen Strongly-Convex Strongly-Concave}'s generalization bounds.



\begin{figure}
\centering
\begin{subfigure}[t]{.48\textwidth}
  \centering
 \includegraphics[width=.9\linewidth]{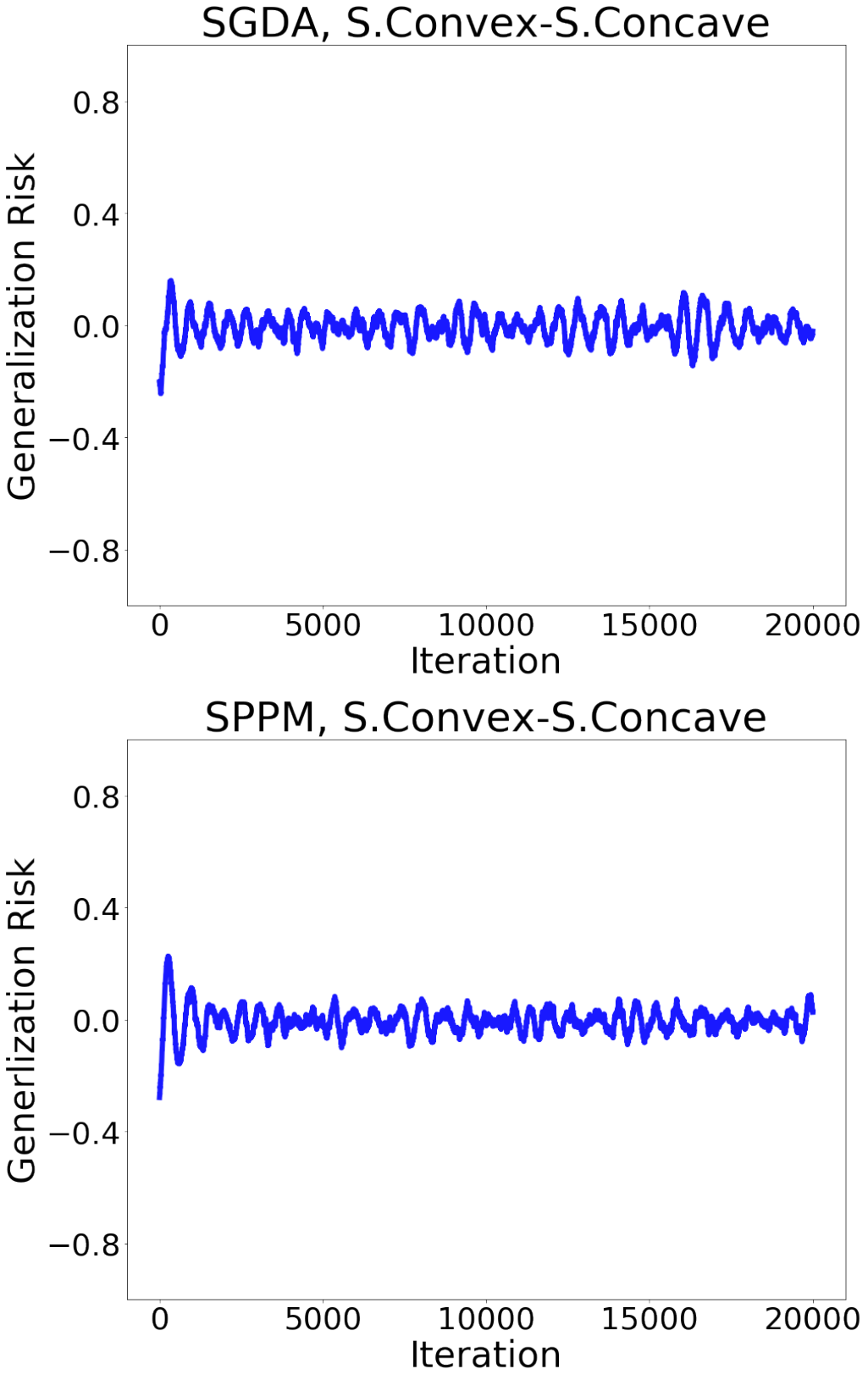}
    \caption{Generalization risk vs. iteration count in the strongly-convex strongly-concave setting optimized by (top) stochastic GDA and (bottom) stochastic PPM.}\label{fig:scsc}
\end{subfigure}\quad%
\begin{subfigure}[t]{.48\textwidth}
  \centering
  \includegraphics[width=.87\linewidth]{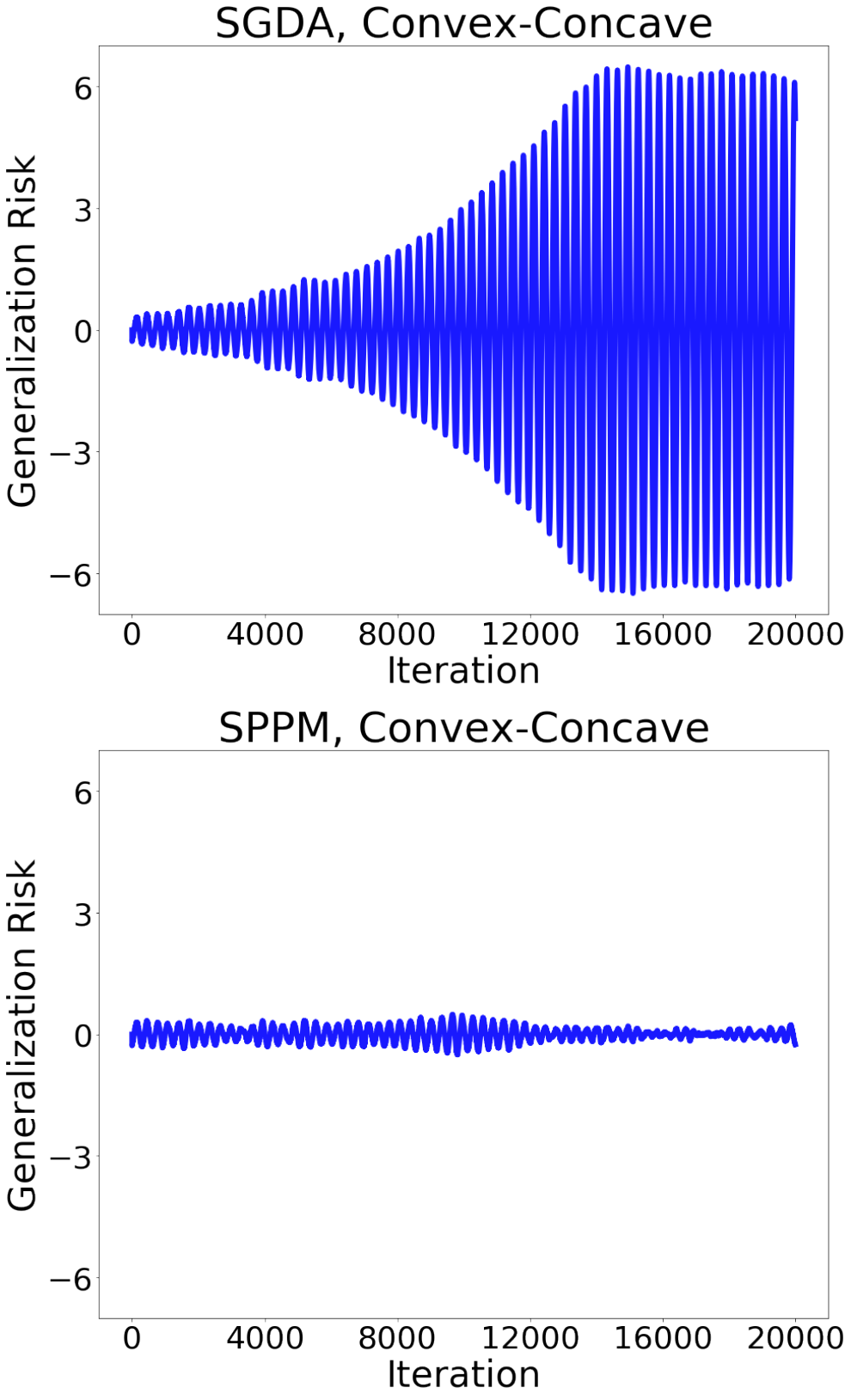}
    \caption{Generalization risk vs. iteration count in the convex concave bilinear setting optimized by (top) stochastic GDA and (bottom) stochastic PPM.}\label{fig:cc}
\end{subfigure}
\caption{Numerical results for convex concave minimax problems}
\label{fig:main_cc}
\end{figure}

Regarding convex concave minimax problems, as suggested by Remark \ref{Remark: convex concave constant stepsize} we considered the following bilinear minimax objective in our experiments:
\begin{equation}\label{Eq: Conv-Conc Experiments}
    f_2(\bw,\btheta;\bz)=\bw^{\top}(\bz-\btheta).
\end{equation}
We constrained the norm of optimization variables as $\Vert\bw\Vert_2,\Vert\btheta\Vert_2\le 100$ which we enforced through projection after every optimization iteration. Similar to the strongly-convex strongly-concave objective \myeqref{Eq: S.Conv-S.Conc Experiments}, for the above minimax objective we have the generalization risk in \myeqref{Eq: Experiments Conv-Conc gen error}  with $\bbE[\bZ]$ and ${\bbE}_S[\bZ]$ being the true and empirical mean vectors.

We optimized the minimax objective \myeqref{Eq: Conv-Conc Experiments} via stochastic and full-batch GDA and PPM algorithms. Figure~\ref{fig:cc} demonstrates the generalization risk evaluated at different iterations of applying stochastic GDA and PPM algorithms. As suggested by Remark \ref{Remark: convex concave constant stepsize}, the generalization risk of stochastic GDA grew exponentially over the first 15,000 iterations before the variables reached the boundary of their feasible sets and then the generalization risk oscillated with a nearly constant amplitude of $6.2$. On the other hand, we observed that the generalization risk of the stochastic PPM algorithm stayed bounded and below $0.5$ for all the 20,000 iterations (Figure~\ref{fig:cc}-bottom). Therefore, our numerical experiments also indicate that while in general convex concave problems the stochastic GDA learner can potentially suffer from a poor generalization performance, the PPM algorithm has a bounded generalization risk as shown by Theorem \ref{Thm: Gen Convex-Concave}.

\subsection{Non-convex Non-concave Problems}


\begin{figure}
\centering
\begin{subfigure}[t]{.48\textwidth}
  \centering
 \includegraphics[width=.9\linewidth]{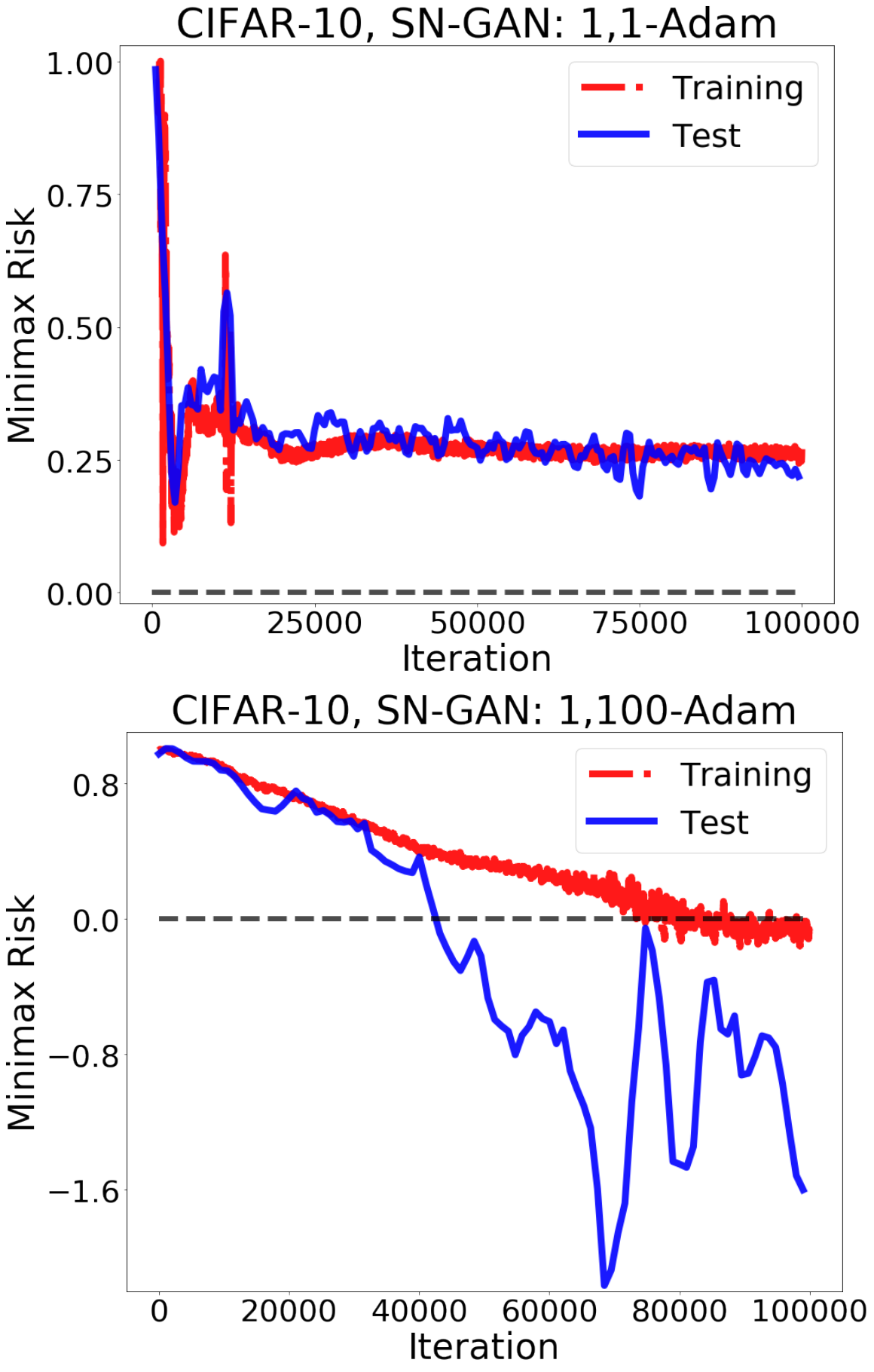}
    \caption{Minimax risk vs. iteration count in the non-convex non-concave SN-GAN problem on CIFAR-10 data optimized by (top) 1,1 Adam descent ascent and (bottom) 1,100 Adam descent ascent}\label{fig:cifar}
\end{subfigure}\quad%
\begin{subfigure}[t]{.48\textwidth}
  \centering
  \includegraphics[width=.87\linewidth]{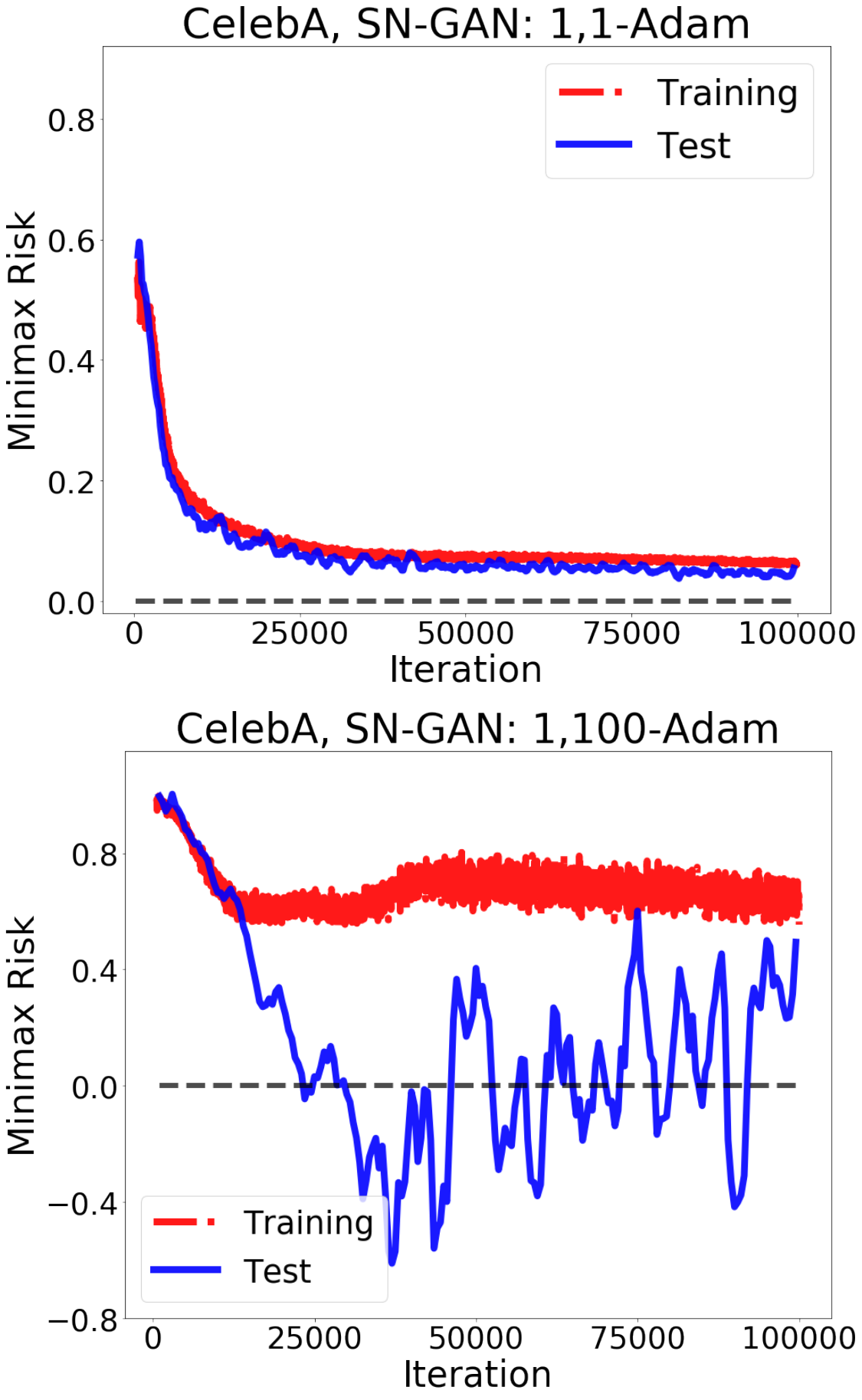}
    \caption{Minimax risk vs. iteration count in the non-convex non-concave SN-GAN problem on CelebA data optimized by (top) 1,1 Adam descent ascent and (bottom) 1,100 Adam descent ascent.}\label{fig:celeba}
\end{subfigure}
\caption{Numerical results for non-convex non-concave minimax problems}
\label{fig:ncnc}
\end{figure}

To numerically analyze generalization in general non-convex non-concave minimax problems, we experimented the performance of simultaneous and non-simultaneous optimization algorithms in training GANs. In our GAN experiments, we considered the standard architecture of DC-GANs \citep{radford2015unsupervised} with 4-layer convolutional neural net generator and discriminator functions. For the minimax objective, we used the formulation of vanilla GAN \citep{goodfellow2014generative} that is 
\begin{equation}
    f(\bw,\btheta;\bz) = \log(D_{\bw}(\bz)) + \bbE_{\boldsymbol{\nu}}\bigl[\log(1-D_{\bw}(G_{\btheta}(\boldsymbol{\nu})))\bigr].  
\end{equation}
For computing the above objective, we used Monte-Carlo simulation using $100$ fresh latent samples $\boldsymbol{\nu}_i\sim\mathcal{N}(\mathbf{0},I_{\tiny r=128})$ to approximate the expected value over generator's latent variable $\boldsymbol{\nu}$ at every optimization step. We followed all the experimental details from \cite{gulrajani2017improved}'s standard implementation of DC-GAN. Furthermore, we applied spectral normalization \citep{miyato2018spectral} to regularize the discriminator function and assist reaching a near optimal solution for discriminator via boundedly many iterations needed for non-simultaneous optimization methods. We trained the spectrally-normalized GAN (SN-GAN) problem over CIFAR-10 \citep{krizhevsky2009learning} and CelebA \citep{liu2018large} datasets. We divided the CIFAR-10 and CelebA datasets to 50,000, 160,000 training and 10,000, 40,000 test samples, respectively.

To optimize the minimax risk function, we used the standard Adam algorithm \citep{kingma2014adam} with batch-size $100$. For simultaneous optimization algorithms we applied 1,1 Adam descent ascent with the parameters $\operatorname{lr}=10^{-4},\, \beta_1=0.5,\, \beta_2=0.9$ for both minimization and maximization updates. To apply a non-simultaneous algorithm, we used 100 Adam maximization steps per minimization step and increased the maximization learning rate to  5$\times 10^{-4}$. We ran each GAN experiment for $T=$100,000 iterations.


Figure~\ref{fig:ncnc} shows the estimates of the empirical and true minimax risks in the CIFAR-10 and CelebA experiments, respectively. We used $2000$ randomly-selected samples from the training and test sets for every estimation task. As seen in Figure~\ref{fig:ncnc}'s plots, for the experiments applying simultaneous 1,1 Adam optimization the empirical minimax risk generalizes properly from training to test samples  (Figure~\ref{fig:ncnc}-top). In contrast, in both the experiments with non-simultaneous methods after 30,000 iterations the empirical minimax risk suffers from a considerable generalization gap from the true minimax risk (Figure~\ref{fig:ncnc}-bottom). The gap between the training and test minimax risks grew between iterations 30,000-60,000. The test minimax risk fluctuated over the subsequent iterations, which could be due to the insufficiency of 100 Adam ascent steps to follow the optimal discriminator solution at those iterations.  

The numerical results of our GAN experiments suggest that non-simultaneous algorithms which attempt to fully solve the maximization subproblem at every iteration can lead to large generalization errors. On the other hand, standard simultaneous algorithms used for training GANs enjoy a bounded generalization error which can help the training process find a model with nice generalization properties. We defer 
further experimental results to the supplementary document.

\bibliography{biblio}
\begin{appendices}
\section{Additional Numerical Results}
\subsection{Convex Concave Minimax Settings}

Here, we provide the results of the numerical experiments discussed in the main text for full-batch GDA and PPM algorithms as well as stochastic and full-batch GDmax algorithms. Note that in these experiments we use the same minimax objective and hyperparameters mentioned in the main text. Figure~\ref{fig:gda} shows the generalization risk in our experiments for the GDA algorithm. As seen in Figure~\ref{fig:gda} (right), the results for full-batch and stochastic GDA algorithms in the bilinear convex concave case look similar, with the only exception that the generalization risk in the full-batch case reached a slightly higher amplitude of 7.8. On the other hand, in the strongly-convex strongly-concave case, full-batch GDA demonstrated a vanishing generalization risk, whereas stochastic GDA could not reach below an amplitude of 0.2.

Figure \ref{fig:ppm} shows the results of our experiments for full-batch PPM. Observe that the generalization risk in both cases decreases to reach smaller values than those for stochastic PPM. Finally, Figures \ref{fig:gdmax} and \ref{fig:sgdmax} include the results for ful-batch and stochastic GDmax algorithms. With the exception of the full-batch GDmax case for the bilinear objective (Figure~\ref{fig:gdmax}-right), in all the other cases the generalization risk did not grow during the optimization, which is comparable to our results in the GDA experiments.

\subsection{Non-convex Non-concave Minimax Settings}
Here, we provide the image samples generated by the trained GANs discussed in the main text. Figure~\ref{fig:cifar_pics} shows the CIFAR-10 samples generated by the simultaneous 1,1 Adam training (Figure~\ref{fig:cifar_pics}-left) and non-simultaneous 1,100-Adam optimization (Figure~\ref{fig:cifar_pics}-right). While we observed that the simultaneous training experiment generated qualitatively sharper samples, the non-simultaneous optimization did not lead to any significant training failures. However, as we discussed in the main text the generalization risk in the non-simultaneous training was significantly larger than that of simultaneous training.  Figure~\ref{fig:celeba_pics} shows the generated images in the CelebA experiments, which are qualitatively comparable between the two training algorithms. However,  as discussed in the text the trained discriminator had a harder task in classifying the training samples from the generated samples than in classifying the test samples from the generated samples, suggesting a potential overfitting of the training samples in the non-simultaneous training experiment.

\begin{figure}[t]
    \centering
    \includegraphics[width=.9\textwidth]{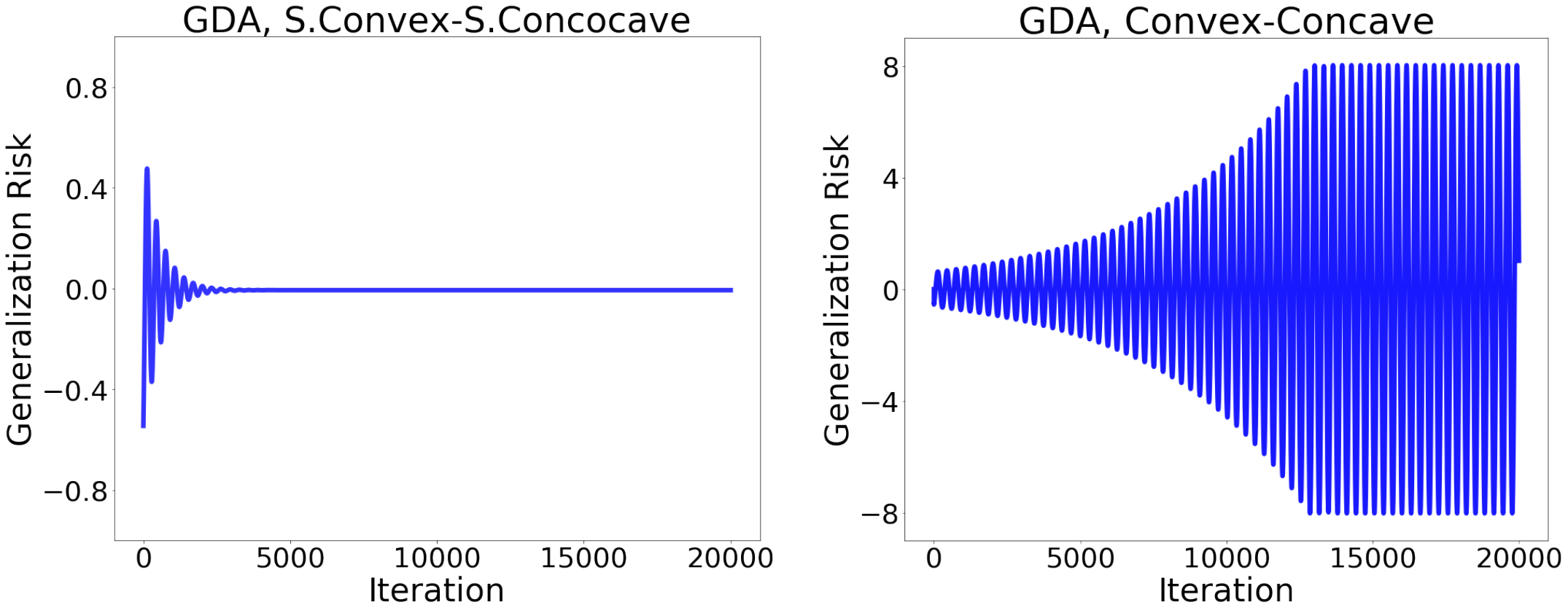}
    \caption{Generalization risk vs. iteration count of full-batch GDA optimization in the (Left) strongly-convex strongly-concave setting and (Right) bilinear  convex concave setting.}\label{fig:gda}
\end{figure}

\begin{figure}[t]
    \centering
    \includegraphics[width=.9\textwidth]{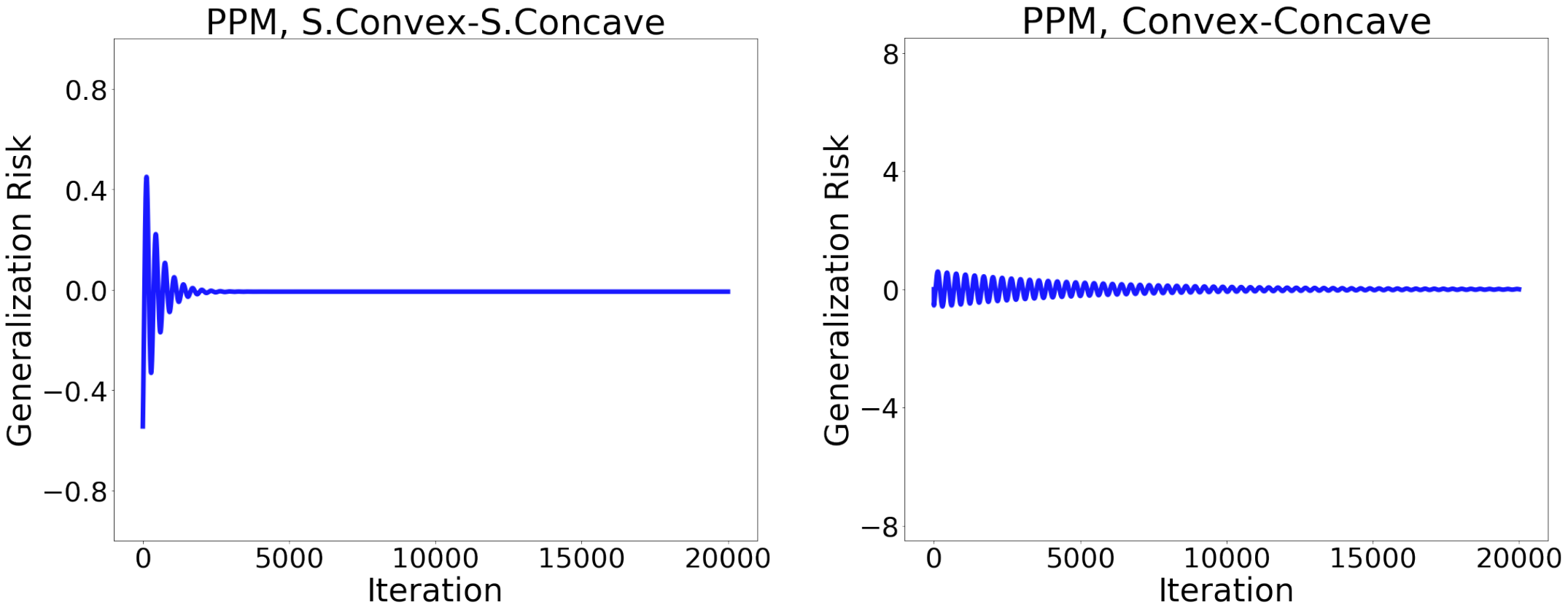}
    \caption{Generalization risk vs. iteration count of full-batch PPM optimization in the (Left) strongly-convex strongly-concave setting and (Right) bilinear  convex concave setting.}\label{fig:ppm}
\end{figure}

\begin{figure}[t]
    \centering
    \includegraphics[width=.9\textwidth]{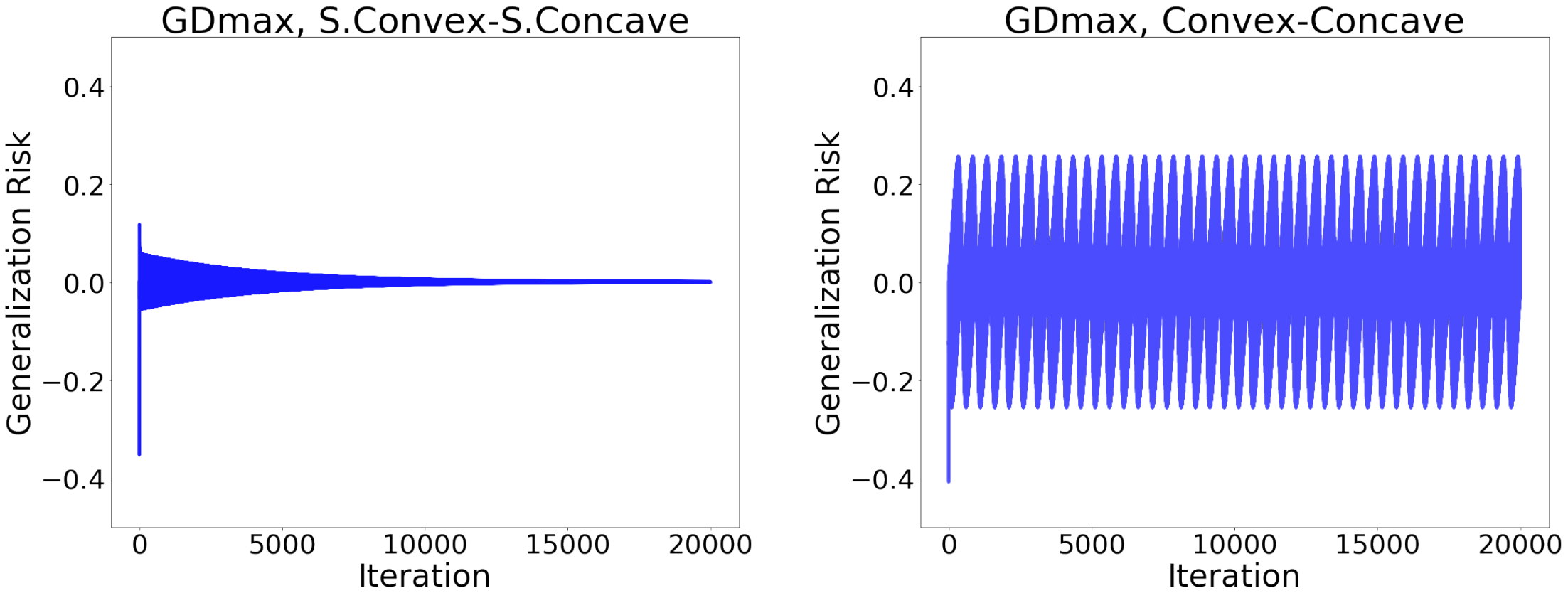}
    \caption{Generalization risk vs. iteration count of full-batch GDmax optimization in the (Left) strongly-convex strongly-concave setting and (Right) bilinear  convex concave setting.}\label{fig:gdmax}
\end{figure}

\begin{figure}[t]
    \centering
    \includegraphics[width=.9\textwidth]{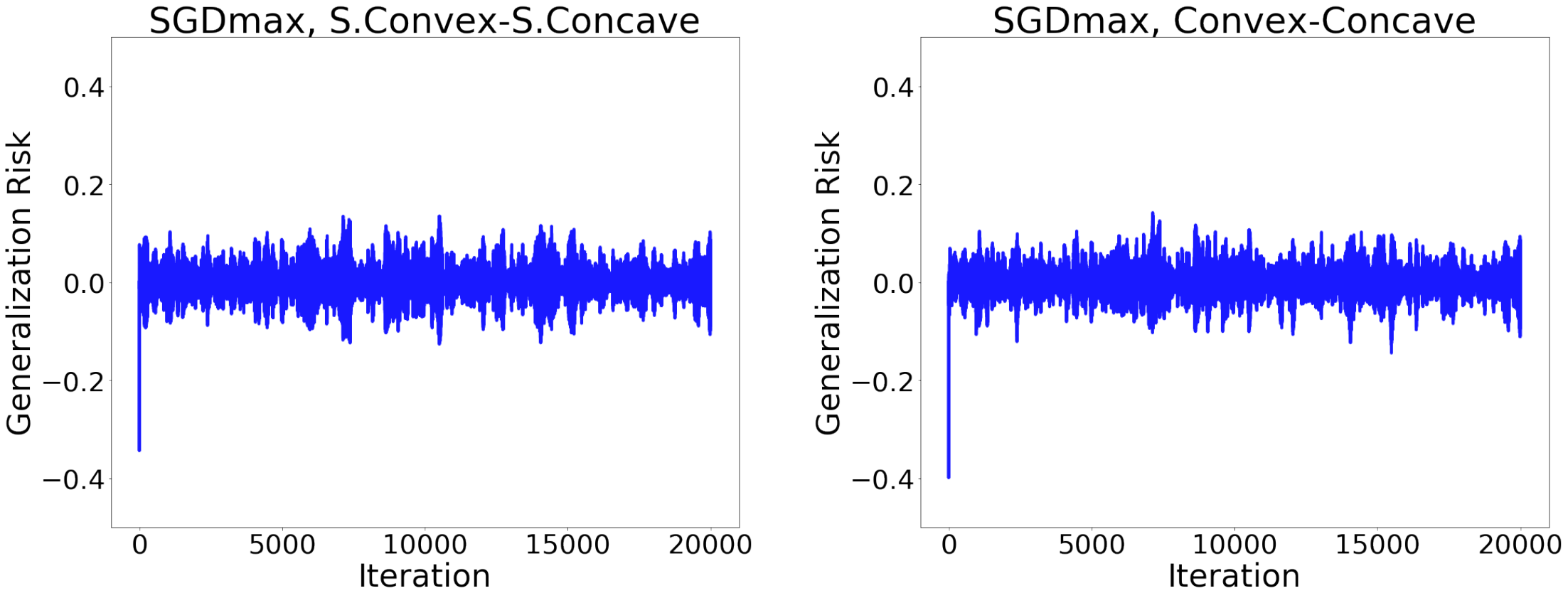}
    \caption{Generalization risk vs. iteration count of stochastic GDmax optimization in the (Left) strongly-convex strongly-concave setting and (Right) bilinear  convex concave setting.}\label{fig:sgdmax}
\end{figure}

\begin{figure}[t]
    \centering
    \includegraphics[width=.9\textwidth]{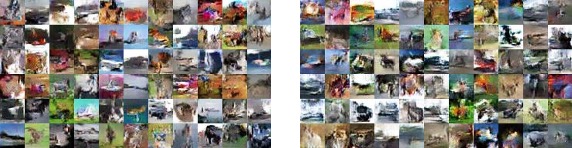}
    \caption{SN-GAN generated pictures in the CIFAR-10 experiments for (Left) simultaneous 1,1-Adam training (Right) non-simultaneous 1,100-Adam training.}\label{fig:cifar_pics}
\end{figure}

\begin{figure}[t]
    \centering
    \includegraphics[width=.9\textwidth]{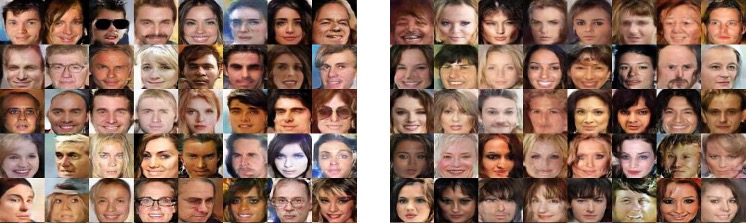}
    \caption{SN-GAN generated pictures in the CelebA-10 experiments for (Left) simultaneous 1,1-Adam training (Right) non-simultaneous 1,100-Adam training.}\label{fig:celeba_pics}
\end{figure}

\section{Proofs}

\subsection{The Expansivity Lemma for Minimax Problems}
We will apply the following lemma to analyze the stability of gradient-based methods. We call an update rule $G$ $\gamma$-expansive if for every $\bw,\bw'\in\mathcal{W},\btheta,\btheta'\in\Theta$ we have
\begin{equation}
    \Vert G(\bw,\btheta) - G(\bw',\btheta') \Vert_2 \le \gamma \sqrt{\Vert\bw-\bw' \Vert^2_2 + \Vert\btheta-\btheta' \Vert^2_2}.
\end{equation}
\begin{lemma}\label{Lemma: Expansive Minimax}
Consider the GDA and PPM updates for the following minimax problem 
\begin{equation}\label{Appendix: Eq: General Minimax Problem}
    \min_{\bw\in\mathcal{W}}\; \max_{\btheta\in \Theta}\; f(\bw,\btheta),
\end{equation}
where we assume objective $f(\bw,\btheta)$ satisfies Assumptions 1 and 2. Then,
\begin{enumerate}[wide,topsep=1pt,itemsep=0pt, labelwidth=!,labelindent=0pt]
    \item For a non-convex non-concave minimax problem, $G_{\gda}$ is $(1+\ell\max\{\alpha_w,\alpha_{\theta}\})$-expansive. Assuming $\eta<\frac{1}{\ell}$, $G_{\ppm}$ will be $1/(1-\ell\eta)$-expansive.
    \item For a convex concave minimax problem with $\alpha_w=\alpha_{\theta}$, $G_{\gda}$ is $\sqrt{1+\ell^2\alpha_w^2}$-expansive and $G_{\ppm}$ will be $1$-expansive.
    \item For a $\mu$-strongly-convex strongly-concave minimax problem, given that $\alpha_w=\alpha_{\theta}\le \frac{2\mu}{\ell^2}$, $G_{\gda}$ is $(1-\alpha_w\mu+\alpha_w^2\ell^2/2)$-expansive and $G_{\ppm}$ will be $1/(1+\mu\eta)$-expansive.
\end{enumerate}
\end{lemma}
\begin{proof}
In Case 1 with non-convex non-concave minimax objective, $f$'s smoothness property implies that for every $(\bw,\btheta)$ and $(\bw',\btheta')$:
\begin{align*}
   \big\Vert G_{\gda}(\begin{bmatrix}
    \bw \\ 
    \btheta
    \end{bmatrix}) - G_{\gda}(\begin{bmatrix}
    \bw' \\ 
    \btheta'
    \end{bmatrix}) \big\Vert &= \big\Vert \begin{bmatrix}
    \bw - \bw' - \alpha_w(\nabla_{\bw}f(\bw,\btheta) - \nabla_{\bw}f(\bw',\btheta')) \\ 
    \btheta - \btheta' +  \alpha_{\theta}(\nabla_{\btheta}f(\bw,\btheta) - \nabla_{\btheta}f(\bw',\btheta'))
    \end{bmatrix}  \big\Vert \\
    &\le \big\Vert \begin{bmatrix}
    \bw - \bw'  \\ 
    \btheta - \btheta'
    \end{bmatrix}  \big\Vert + \big\Vert \begin{bmatrix}
    \alpha_w(\nabla_{\bw}f(\bw,\btheta) - \nabla_{\bw}f(\bw',\btheta')) \\ 
      \alpha_{\theta}(\nabla_{\btheta}f(\bw,\btheta) - \nabla_{\btheta}f(\bw',\btheta'))
    \end{bmatrix}  \big\Vert \\
    &\le (1+\ell\max\{\alpha_w,\alpha_{\theta}\})\big\Vert \begin{bmatrix}
    \bw   \\ 
    \btheta 
    \end{bmatrix} - \begin{bmatrix}
     \bw'  \\ 
     \btheta'
    \end{bmatrix}  \big\Vert,\numberthis
\end{align*}
which completes the proof for the GDA update. For the proximal operator, note that given $\eta\le \frac{1}{\ell}$ the proximal optimization reduces to optimizing a strongly-convex strongly-concave minimax problem with a unique saddle solution and therefore at $(\bw_{\ppm}, \btheta_{\ppm})=G_{\ppm}(\bw, \btheta)$ we have
\begin{equation}
    \bw_{\ppm} - \bw =  \eta \nabla_w f(\bw_{\ppm}, \btheta_{\ppm}), \quad \btheta - \btheta_{\ppm} =  \eta \nabla_{\theta} f(\bw_{\ppm}, \btheta_{\ppm}).
\end{equation}
As a result, we have
\begin{align*}
   &\big\Vert G_{\ppm}(\begin{bmatrix}
    \bw \\ 
    \btheta
    \end{bmatrix}) - G_{\ppm}(\begin{bmatrix}
    \bw' \\ 
    \btheta'
    \end{bmatrix}) \big\Vert  \\
    = &\big\Vert \begin{bmatrix}
    \bw - \bw' + \eta(\nabla_{\bw}f(G_{\ppm}(\bw,\btheta)) - \nabla_{\bw}f(G_{\ppm}(\bw',\btheta'))) \\ 
    \btheta - \btheta' -  \eta(\nabla_{\btheta}f(G_{\ppm}(\bw,\btheta)) - \nabla_{\btheta}f(G_{\ppm}(\bw',\btheta'))
    \end{bmatrix}  \big\Vert \\
    \le &\big\Vert \begin{bmatrix}
    \bw - \bw'  \\ 
    \btheta - \btheta'
    \end{bmatrix}  \big\Vert + \big\Vert \begin{bmatrix}
    \eta(\nabla_{\bw}f(G_{\ppm}(\bw,\btheta)) - \nabla_{\bw}f(G_{\ppm}(\bw',\btheta'))) \\ 
      \eta(\nabla_{\btheta}f(G_{\ppm}(\bw,\btheta)) - \nabla_{\btheta}f(G_{\ppm}(\bw',\btheta'))
    \end{bmatrix}  \big\Vert \\
    \le & \big\Vert \begin{bmatrix}
    \bw   \\ 
    \btheta 
    \end{bmatrix} - \begin{bmatrix}
     \bw'  \\ 
     \btheta'
    \end{bmatrix}  \big\Vert + \frac{\eta}{\ell} \big\Vert G_{\ppm}(\bw,\btheta) -G_{\ppm}(\bw',\btheta') \big\Vert.\numberthis 
\end{align*}
The final result of the above inequalities implies that
\begin{equation}
     (1-\frac{\eta}{\ell}) \big\Vert G_{\ppm}(\bw,\btheta) -G_{\ppm}(\bw',\btheta') \big\Vert \le \big\Vert \begin{bmatrix}
    \bw   \\ 
    \btheta 
    \end{bmatrix} - \begin{bmatrix}
     \bw'  \\ 
     \btheta'
    \end{bmatrix}  \big\Vert,
\end{equation}
which completes the proof for the case of non-convex non-concave case.

For convex-concave objectives,
the proof is mainly based on the monotonicity of convex concave objective's gradients \citep{rockafellar1976monotone}, implying that for every $\bw,\bw',\btheta,\btheta'$:
\begin{equation}
    \bigl( \begin{bmatrix}
    \bw   \\ 
    \btheta 
    \end{bmatrix} - \begin{bmatrix}
     \bw'  \\ 
     \btheta'
    \end{bmatrix} \bigr)^T \bigl( \begin{bmatrix}
    \nabla_{\bw} f(\bw,\btheta)   \\ 
    -\nabla_{\btheta} f(\bw,\btheta) 
    \end{bmatrix} - \begin{bmatrix}
     \nabla_{\bw} f(\bw',\btheta')   \\ 
    -\nabla_{\btheta} f(\bw',\btheta')
    \end{bmatrix}\bigr) \ge 0.
\end{equation}
As shown by \cite{rockafellar1976monotone}, the above property implies that the proximal operator for a convex-concave minimax objective will also be monotone and $1$-expansive for any positive choice of $\eta$. For the GDA update, note that due to the monotonicity property
\begin{align*}
    &\big\Vert G_{\gda}(\begin{bmatrix}
    \bw \\ 
    \btheta
    \end{bmatrix}) - G_{\gda}(\begin{bmatrix}
    \bw' \\ 
    \btheta'
    \end{bmatrix}) \big\Vert^2_2 \\
    =\, &\big\Vert \begin{bmatrix}
    \bw-\bw' \\ 
    \btheta-\btheta'
    \end{bmatrix} \big\Vert_2^2 - 2 \alpha_w \begin{bmatrix}
    \bw -\bw'  \\ 
    \btheta - \btheta'
    \end{bmatrix}^T \begin{bmatrix}
    \nabla_{\bw} f(\bw,\btheta) - \nabla_{\bw} f(\bw',\btheta')   \\ 
    -\nabla_{\btheta} f(\bw,\btheta) + \nabla_{\btheta} f(\bw',\btheta')
    \end{bmatrix} \\
    &\quad +\alpha_w^2 \big\Vert \begin{bmatrix}
    \nabla_{\bw} f(\bw,\btheta)-\nabla_{\bw} f(\bw',\btheta') \\ 
     \nabla_{\btheta} f(\bw,\btheta)-\nabla_{\btheta} f(\bw',\btheta')
    \end{bmatrix} \big\Vert_2^2 \\
    \le\, &(1+\alpha^2_w\ell^2)\big\Vert \begin{bmatrix}
    \bw-\bw' \\ 
    \btheta-\btheta'
    \end{bmatrix} \big\Vert_2^2,\numberthis
\end{align*}
which results in the following inequality and completes the proof for the convex-concave case:
\begin{align*}
    \big\Vert G_{\gda}(\begin{bmatrix}
    \bw \\ 
    \btheta
    \end{bmatrix}) - G_{\gda}(\begin{bmatrix}
    \bw' \\ 
    \btheta'
    \end{bmatrix}) \big\Vert_2 
    \le\, \sqrt{1+\alpha^2_w\ell^2}\,\big\Vert \begin{bmatrix}
    \bw-\bw' \\ 
    \btheta-\btheta'
    \end{bmatrix} \big\Vert_2.\numberthis
\end{align*}
Finally, for the strongly-convex strongly-concave case, note that $\tilde{f}(\bw,\btheta)=f(\bw,\btheta) +\frac{\mu}{2}(\Vert\btheta \Vert^2- \Vert\bw \Vert^2)$ will be convex-concave and hence the proximal update $(\bw_{\ppm}, \btheta_{\ppm})=G_{\ppm}(\bw, \btheta)$ will satisfy
\begin{align*}
    \frac{1}{1+\mu\eta}\bw &= \bw_{\ppm} + \frac{\eta}{1+\mu\eta}\nabla_{\bw} \tilde{f}(\bw_{\ppm},\btheta_{\ppm}), \\
    \frac{1}{1+\mu\eta}\btheta &= \btheta_{\ppm} - \frac{\eta}{1+\mu\eta}\nabla_{\btheta} \tilde{f}(\bw_{\ppm},\btheta_{\ppm}),\numberthis
\end{align*}
where the right-hand side follows from the proximal update for $\tilde{f}$ with stepsize ${\eta}/(1+\mu\eta)$ and hence $1$-expansive. Therefore, the proximal update for $f$ will be $1/(1+\mu\eta)$-expansive. Furthemore, for GDA udpates note that
\begin{align*}
    &\big\Vert G_{\gda}(\begin{bmatrix}
    \bw \\ 
    \btheta
    \end{bmatrix}) - G_{\gda}(\begin{bmatrix}
    \bw' \\ 
    \btheta'
    \end{bmatrix}) \big\Vert^2_2 \\
    =\, &(1-\mu\alpha_w)^2\big\Vert \begin{bmatrix}
    \bw-\bw' \\ 
    \btheta-\btheta'
    \end{bmatrix} \big\Vert_2^2 - 2 (1-\mu\alpha_w)\alpha_w \begin{bmatrix}
    \bw -\bw'  \\ 
    \btheta - \btheta'
    \end{bmatrix}^T \begin{bmatrix}
    \nabla_{\bw} \tilde{f}(\bw,\btheta) - \nabla_{\bw} \tilde{f}(\bw',\btheta')   \\ 
    -\nabla_{\btheta} \tilde{f}(\bw,\btheta) + \nabla_{\btheta} \tilde{f}(\bw',\btheta')
    \end{bmatrix} \\
    &\quad +\alpha_w^2 \big\Vert \begin{bmatrix}
    \nabla_{\bw}  \tilde{f}(\bw,\btheta)-\nabla_{\bw}  \tilde{f}(\bw',\btheta') \\ 
     \nabla_{\btheta}  \tilde{f}(\bw,\btheta)-\nabla_{\btheta}  \tilde{f}(\bw',\btheta')
    \end{bmatrix} \big\Vert_2^2 \\
    \le\, &((1-\mu\alpha_w)^2+\alpha^2_w(\ell^2-\mu^2))\big\Vert \begin{bmatrix}
    \bw-\bw' \\ 
    \btheta-\btheta'
    \end{bmatrix} \big\Vert_2^2 \\
    \le\, &(1-2\mu\alpha_w+\alpha^2_w\ell^2)\big\Vert \begin{bmatrix}
    \bw-\bw' \\ 
    \btheta-\btheta'
    \end{bmatrix} \big\Vert_2^2.\numberthis
\end{align*}
Note that the above result finishes the proof because $\sqrt{1-t}\le 1-t/2$ holds for every $t\le 1$, which is based on the lemma's assumption $\alpha_w\le 2\mu/\ell^2$. Also, the last inequality in the above holds since $\tilde{f}$ will be $\sqrt{\ell^2-\mu^2}$-smooth. This is because $f$ is assumed to be $\ell$-smooth, implying that for every $\bw,\bw',\btheta,\btheta'$ we have
\begin{align*}
    \ell^2\big\Vert\begin{bmatrix}
    \bw-\bw' \\ 
    \btheta-\btheta'
    \end{bmatrix} \big\Vert_2^2 &\ge \big\Vert \begin{bmatrix}
    \nabla_{\bw}  f(\bw,\btheta)-\nabla_{\bw}  f(\bw',\btheta') \\ 
     \nabla_{\btheta}  f(\bw,\btheta)-\nabla_{\btheta}  f(\bw',\btheta')
    \end{bmatrix} \big\Vert_2^2 \\
    &= \mu^2\big\Vert\begin{bmatrix}
    \bw-\bw' \\ 
    \btheta-\btheta'
    \end{bmatrix}\big\Vert^2_2 + 2\mu\begin{bmatrix}
    \bw-\bw' \\ 
    \btheta'-\btheta
    \end{bmatrix}^T\begin{bmatrix}
    \nabla_{\bw}  \tilde{f}(\bw,\btheta)-\nabla_{\bw}  \tilde{f}(\bw',\btheta') \\ 
     \nabla_{\btheta}  \tilde{f}(\bw,\btheta)-\nabla_{\btheta}  \tilde{f}(\bw',\btheta')
    \end{bmatrix} \\
     & \quad  + \big\Vert \begin{bmatrix}
    \nabla_{\bw}  \tilde{f}(\bw,\btheta)-\nabla_{\bw}  \tilde{f}(\bw',\btheta') \\ 
     \nabla_{\btheta}  \tilde{f}(\bw,\btheta)-\nabla_{\btheta}  \tilde{f}(\bw',\btheta') 
     \end{bmatrix} \big\Vert^2_2\\
     &\ge \mu^2\big\Vert\begin{bmatrix}
    \bw-\bw' \\ 
    \btheta-\btheta'
    \end{bmatrix} \big\Vert^2_2 + \big\Vert \begin{bmatrix}
    \nabla_{\bw}  \tilde{f}(\bw,\btheta)-\nabla_{\bw}  \tilde{f}(\bw',\btheta') \\ 
     \nabla_{\btheta}  \tilde{f}(\bw,\btheta)-\nabla_{\btheta}  \tilde{f}(\bw',\btheta') 
     \end{bmatrix} \big\Vert^2_2,\numberthis
\end{align*}
where the inequality uses the monotonicity of the gradient operator. The final inequality shows that $\tilde{f}$ will be $\sqrt{\ell^2-\mu^2}$-smooth and hence finishes the proof.
\end{proof}

\subsection{Proof of Theorem 1
}
\begin{thm*}
Assume minimax learner $A$ is $\epsilon$-uniformly stable in minimization. Then, $A$'s expected generalization risk is bounded as $\epsilon_{gen}(A)\le \epsilon$.
\end{thm*}
\begin{proof}
Here, we provide a proof based on standard techniques in stability-based generalization theory \citep{bousquet2002stability}.
To show this theorem, consider two independent datasets $S=(z_1,\ldots,z_n)$ and $S'=(z'_1,\ldots,z'_n)$. Using $S^{(i)}=(z_1,\ldots,z_{i-1},z'_i,z_{i+1},\ldots,z_n)$ to denote the dataset with the $i$th sample replaced with $z'_i$, we will have
\begin{align*}
    \bbE_S\bbE_A[R_S(A_w(S))] &= \bbE_S\bbE_A\biggl[\frac{1}{n}\sum_{i=1}^n \max_{\btheta\in\Theta} f(A_w(S),\btheta;z_i) \biggr] \\
    &= \bbE_S\bbE_{S'}\bbE_A\biggl[\frac{1}{n}\sum_{i=1}^n \max_{\btheta\in\Theta} f(A_w(S^{(i)}),\btheta;z'_i) \biggr] \\
    &= \bbE_S\bbE_{S'}\bbE_A\biggl[\frac{1}{n}\sum_{i=1}^n \max_{\btheta\in\Theta} f(A_w(S),\btheta;z'_i) \biggr] + \zeta \\
    &= \bbE_S\bbE_A[R(A_w(S))] + \zeta.\numberthis
\end{align*}
In the above, $\zeta$ is defined as
\begin{equation}
    \zeta := \bbE_S\bbE_{S'}\bbE_A\biggl[\frac{1}{n}\sum_{i=1}^n \bigl[ \max_{\btheta\in\Theta} f(A_w(S^{(i)}),\btheta;z'_i) -  \max_{\btheta'\in\Theta} f(A_w(S),\btheta';z'_i) \bigr] \biggr].
\end{equation}
Note that due to the uniform stability assumption for every data point $z$ and datasets $S,S'$ with only one different sample we have
\begin{align*}
    \max_{\btheta\in\Theta} f(A_w(S),\btheta;z) -  \max_{\btheta'\in\Theta} f(A_w(S'),\btheta';z) \le  \max_{\btheta\in\Theta}\bigl\{ f(A_w(S),\btheta;z)- f(A_w(S'),\btheta;z)\bigr\} \le \epsilon.\numberthis
\end{align*}
Therefore, replacing the order of $S,S'$ in the above inequality we obtain
\begin{align*}
    \bigl\vert \max_{\btheta\in\Theta} f(A_w(S),\btheta;z) -  \max_{\btheta'\in\Theta} f(A_w(S'),\btheta';z) \bigr\vert \le  \epsilon.\numberthis
\end{align*}
As a result, we conclude that $\vert \zeta \vert\le \epsilon$ which shows that
\begin{equation}
    \big\vert \bbE_S\bbE_A[R_S(A_w(S))] -  \bbE_S\bbE_A[R(A_w(S))] \big\vert \le \epsilon.
\end{equation}
The proof is hence complete.
\end{proof}

\subsection{Proof of Theorem 2
}
Note that in the following discussion we define PPmax as a proximal point method which fully optimizes the maximization variable at every iteration with the following update rule:
\begin{equation}
    G_{\ppmax}(\begin{bmatrix}
    \bw \\ 
    \btheta
    \end{bmatrix}) :=  \underset{\widetilde{\bw}\in\mathcal{W}}{\arg\!\min}\:\underset{\widetilde{\btheta}\in\Theta}{\arg\!\max}\: \bigl\{ f(\widetilde{\bw},\widetilde{\btheta})\nonumber  +\frac{1}{2\eta_{w}}\Vert \widetilde{\bw} - \bw \Vert^2_2\bigr\}.
\end{equation}
\begin{thm*}
Let minimax learning objective $f(\cdot,\cdot;\bz)$ be $\mu$-strongly convex strongly-concave and satisfy Assumption 2 for every $\bz$. Assume that Assumption 1 holds for convex-concave $\widetilde{f}(\bw,\btheta;\bz):=f(\bw,\btheta;\bz)+\frac{\mu}{2}(\Vert \btheta\Vert^2_2-\Vert \bw\Vert^2_2)$ and every $\bz$. 
Then,
\begin{enumerate}[wide,topsep=1pt,itemsep=0pt, labelwidth=!,labelindent=0pt]
    \item Full-batch and Stochastic GDA and GDmax with constant stepsize $\alpha_w=\alpha_{\theta}\le \frac{2\mu}{\ell^2}$ for $T$ iterations will satisfy
    \begin{equation}
        \epsilon_{\gen}(\gda),\epsilon_{\gen}(\sgda)\le \frac{2LL_w}{(\mu-\frac{\alpha_w\ell^2}{2})n},\quad \epsilon_{\gen}(\gdmax),\epsilon_{\gen}(\sgdmax)\le \frac{2L^2_w}{\mu n}.
    \end{equation}
    \item Full-batch and stochastic PPM and PPmax with constant parameter $\eta$ for $T$ iterations will satisfy
    \begin{equation}
        \epsilon_{\gen}(\ppm),\epsilon_{\gen}(\sppm)\le \frac{2LL_w}{\mu n}, \quad \epsilon_{\gen}(\ppmax),\epsilon_{\gen}(\sppmax)\le \frac{2L^2_w}{\mu n}.
    \end{equation}
\end{enumerate}
\end{thm*}
\begin{proof}
We start by proving the following lemmas.
\begin{lemma}[Growth Lemma]\label{Lemma: Growth}
Consider two sequences of updates $G_1,\ldots,G_T$ and $G'_1,\ldots.G'_t$ with the same starting point $\bw_0=\bw'_0,\btheta_0=\btheta'_0$. We define $\delta_t:= \sqrt{\Vert \bw_t-\bw'_t \Vert^2+\Vert \btheta_t-\btheta'_t\Vert^2}$. Then, if $G_T,G'_T$ is $\xi$-expansive we have $\delta_{t+1}\le \xi \delta_t$ for identical $G_t=G'_t$, and in general we have
\begin{equation}
    \delta_{t+1}\le \min\{\xi,1\}\delta_t + \sup_{\bw,\btheta}\{ \Vert [\bw,\btheta] - G_t([\bw,\btheta]) \Vert \} + \sup_{\bw,\btheta}\{ \Vert [\bw,\btheta] - G'_t([\bw,\btheta]) \Vert \}.
\end{equation}
Furthermore, for any constant $r$ we have 
\begin{equation}
    \delta_{t+1}\le \xi\delta_t + \sup_{\bw,\btheta}\{ \Vert r[\bw,\btheta] - G_t([\bw,\btheta]) \Vert \} + \sup_{\bw,\btheta}\{ \Vert r[\bw,\btheta] - G'_t([\bw,\btheta]) \Vert \}.
\end{equation}
Finally, if $G_t=G+\tilde{G}_t$ and $G'_t= G + \tilde{G}'_t$ for $\xi_0$-expansive $G$ and $\xi_1$-expansive $\tilde{G}_t$ and $\tilde{G}'_t$, then for any constant $r$ we have
\begin{equation}
    \delta_{t+1}\le (\xi_0+\xi_1)\delta_t + \sup_{\bw,\btheta}\{ \Vert r[\bw,\btheta] - \tilde{G}_t([\bw,\btheta]) \Vert \} + \sup_{\bw,\btheta}\{ \Vert r[\bw,\btheta] - \tilde{G}'_t([\bw,\btheta]) \Vert \}.
\end{equation}
\end{lemma}
\begin{proof}
The first part of the theorem is a direct consequence of the definition of $\xi$-expansive operators. For the second part, note that
\begin{align*}
    \delta_{t+1} &= \Vert G_t([\bw_t,\btheta_t]) - G'_t([\bw'_t,\btheta'_t])  \Vert \\
    &= \Vert G_t([\bw_t,\btheta_t]) - [\bw_t,\btheta_t] + [\bw_t,\btheta_t] -[\bw'_t,\btheta'_t] +[\bw'_t,\btheta'_t]  - G'_t([\bw'_t,\btheta'_t])  \Vert \\
    &\le \Vert G_t([\bw_t,\btheta_t]) - [\bw_t,\btheta_t]\Vert + \Vert  [\bw_t,\btheta_t] -[\bw'_t,\btheta'_t] \Vert  + \Vert [\bw'_t,\btheta'_t]  - G'_t([\bw'_t,\btheta'_t]) \Vert \\
    &\le \delta_t + \sup_{\bw,\btheta}\{ \Vert [\bw,\btheta] - G_t([\bw,\btheta]) \Vert \} + \sup_{\bw,\btheta}\{ \Vert [\bw,\btheta] - G'_t([\bw,\btheta]) \Vert \}.\numberthis
\end{align*}
In addition, we can bound $\delta_{t+1}$ as
\begin{align*}
    \delta_{t+1} &= \Vert G_t([\bw_t,\btheta_t]) - G'_t([\bw'_t,\btheta'_t])  \Vert \\
    &= \Vert G_t([\bw_t,\btheta_t]) - G_t([\bw'_t,\btheta'_t])  + G_t([\bw'_t,\btheta'_t])  - G'_t([\bw'_t,\btheta'_t])  \Vert  \\
    &\le \Vert G_t([\bw_t,\btheta_t]) - G_t([\bw'_t,\btheta'_t])\Vert + \Vert G_t([\bw'_t,\btheta'_t])  - G'_t([\bw'_t,\btheta'_t])  \Vert \\
    &\le \xi \delta_t + \Vert G_t([\bw'_t,\btheta'_t])  - r[\bw'_t,\btheta'_t]  \Vert + \Vert  r[\bw'_t,\btheta'_t]  - G'_t([\bw'_t,\btheta'_t]) \Vert \\
    &\le \xi\delta_t + \sup_{\bw,\btheta}\{ \Vert r[\bw,\btheta] - G_t([\bw,\btheta]) \Vert \} + \sup_{\bw,\btheta}\{ \Vert r[\bw,\btheta] - G'_t([\bw,\btheta]) \Vert \}.\numberthis
\end{align*}
The above result for general constant $r$ and also combined with the previous result with $r=1$ finishes the proof of the first two parts. For the final segment of the lemma, note that 
\begin{align*}
    \delta_{t+1} &= \Vert G_t([\bw_t,\btheta_t]) - G'_t([\bw'_t,\btheta'_t])  \Vert \\
    &= \Vert G([\bw_t,\btheta_t]) + \tilde{G}_t([\bw_t,\btheta_t]) - G([\bw'_t,\btheta'_t]) - \tilde{G}'_t([\bw'_t,\btheta'_t])  \Vert \\
    &\le \Vert G([\bw_t,\btheta_t])- G([\bw'_t,\btheta'_t])\Vert + \Vert \tilde{G}_t([\bw_t,\btheta_t])  - \tilde{G}'_t([\bw'_t,\btheta'_t])  \Vert
    \\
    &\le \xi_0\delta_t + \Vert \tilde{G}_t([\bw_t,\btheta_t])  - \tilde{G}'_t([\bw'_t,\btheta'_t])  \Vert
    \\
    &\le \xi_0\delta_t + \xi_1\delta_t + \sup_{\bw,\btheta}\{ \Vert r[\bw,\btheta] - \tilde{G}_t([\bw,\btheta]) \Vert \} + \sup_{\bw,\btheta}\{ \Vert r[\bw,\btheta] - \tilde{G}'_t([\bw,\btheta]) \Vert \}.\numberthis
\end{align*}
In the above equations, the last line follows from the second part of the lemma which finishes the proof.
\end{proof}
In order to show the Theorem for SGDA updates, note that given two datsets $S,S'$ of size $n$ with only one different sample at every iteration of stochastic GDA the update rule will be the same with probability $1-1/n$ and with probability $1/n$ we have two different $(1-\alpha_w\mu +\alpha_w^2\ell^2/2)$-expansive operators both of which satisfy
\begin{equation}
 \sup_{\bw,\btheta}\{ \Vert (1-\alpha_w\mu)[\bw,\btheta] - G_{\sgda}([\bw,\btheta]) \Vert \} \le L \alpha_w.
\end{equation}
The above inequality holds, because $\tilde{f}$ is assumed to be continuously differentiable and $L$-Lipschitz. As a result, Lemmas \ref{Lemma: Growth},\ref{Lemma: Expansive Minimax} together with the law of total probability imply that the expected norm of $\delta^{\sgda}_t= \sqrt{\Vert \bw_t-\bw'_t \Vert^2+\Vert \btheta_t-\btheta'_t\Vert^2}$ for the SGDA updates applied to the two datasets will satisfy
\begin{align*}
    \mathbb{E}[\delta^{\sgda}_{t+1}] &\le (1-\frac{1}{n})(1-\alpha_w\mu +\frac{\alpha_w^2\ell^2}{2})\mathbb{E}[\delta^{\sgda}_{t}] + \frac{1}{n}\bigl( (1-\alpha_w\mu +\frac{\alpha_w^2\ell^2}{2})\mathbb{E}[\delta^{\sgda}_{t}] + 2\alpha_w L \bigr) \\
    &= (1-\alpha_w\mu +\frac{\alpha_w^2\ell^2}{2})\mathbb{E}[\delta^{\sgda}_{t}] + \frac{2\alpha_w L}{n}. \numberthis
\end{align*}
Note that in the above upper-bound $1-\frac{1}{n}$ is the probability that the stochastic GDA algorithm chooses a shared sample between the two datasets and $\frac{1}{n}$ is the probability of picking the index of the different sample. 

Similarly, the update rule for the full-batch GDA algorithm can be written as the sum of the updates for the shared samples, i.e.,  $\sum_{i=1}^{n-1} \frac{1}{n}G_{\gda}([\bw,\btheta];\bz_i)$, and the different sample $\bz_n$'s update $ \frac{1}{n}G_{\gda}([\bw,\btheta];\bz_n)$. As a result, the last part of Lemma \ref{Lemma: Growth} together with Lemma \ref{Lemma: Expansive Minimax} implies that
\begin{align*}
    \delta^{\gda}_{t+1} &\le (1-\frac{1}{n}+\frac{1}{n})(1-\alpha_w\mu +\frac{\alpha_w^2\ell^2}{2})\delta^{\gda}_{t} + \frac{1}{n}(  2\alpha_w L) \\
    &= (1-\alpha_w\mu +\frac{\alpha_w^2\ell^2}{2})\mathbb{E}[\delta^{\gda}_{t}] + \frac{2\alpha_w L}{n}, \numberthis
\end{align*}
which is the same bound we derived for stochastic GDA. Therefore, given that $\delta_0=0$, for SGDA updates we have
\begin{align*}
    \mathbb{E}[\delta^{\sgda}_{t}] &\le \frac{2\alpha_w L}{n}\sum_{i=0}^t (1-\alpha_w\mu +\frac{\alpha_w^2\ell^2}{2})^i  \\
    &\le \frac{2\alpha_w L}{n}\sum_{i=0}^\infty (1-\alpha_w\mu +\frac{\alpha_w^2\ell^2}{2})^i \\
    &= \frac{2\alpha_w L}{n(\alpha_w\mu -\frac{\alpha_w^2\ell^2}{2})} \\
    &= \frac{2L}{n(\mu-\frac{\alpha_w\ell^2}{2})}.\numberthis
\end{align*}
Note that $\Vert \bw_t-\bw'_t\Vert \le \delta_t$ and for every $\btheta,\bz$ $f(\bw,\btheta;\bz)-\frac{\mu}{2}\Vert \bw\Vert_2^2$ is $L_w$-Lipschitz in $\bw$. As a result, the SGDA algorithm applied for $T$ iterations will be $(2LL_w/n(\mu-{\alpha_w\ell^2}/{2}))$-uniformly stable in minimization, and the result follows from Theorem 1. The result for the GDA algorithm will follow from the same steps, since it shares the same growth rule with the SGDA algorithm.

Similarly, the SPPM updates will be $1/(1+\mu\eta)$-expansive due to Lemma \ref{Lemma: Expansive Minimax}. Furthermore, they will satisfy
\begin{equation}
 \sup_{\bw,\btheta}\{ \Vert \frac{1}{1+\eta\mu}[\bw,\btheta] - G_{\sppm}([\bw,\btheta]) \Vert \} \le \frac{L \eta}{1+\eta\mu}.
\end{equation}
The above equation holds, because for a SPPM update $[\bw_{\sppm},\btheta_{\sppm}]=G_{\sppm}([\bw,\btheta])$ at sample $\bz$ we have
\begin{align*}
 &\begin{bmatrix}
 \bw \\
 \btheta 
 \end{bmatrix}
  = \begin{bmatrix}
 (1+\eta\mu)\bw_{\sppm} + \eta\nabla_w \tilde{f}(\bw_{\sppm},\btheta_{\sppm};\bz) \\
 (1+\eta\mu)\btheta_{\sppm} - \eta\nabla_{\theta} \tilde{f}(\bw_{\sppm},\btheta_{\sppm};\bz) 
 \end{bmatrix}\\
 \Rightarrow\;\; &\frac{1}{1+\eta\mu}\begin{bmatrix}
 \bw \\
 \btheta
 \end{bmatrix} - \begin{bmatrix}
  \bw_{\sppm} \\
 \btheta_{\sppm} 
 \end{bmatrix}
  = \frac{\eta}{1+\eta\mu}\begin{bmatrix}
 \nabla_w \tilde{f}(\bw_{\sppm},\btheta_{\sppm};\bz) \\
 -\nabla_{\theta} \tilde{f}(\bw_{\sppm},\btheta_{\sppm};\bz). 
 \end{bmatrix}\numberthis
\end{align*}
Therefore, applying the law of total probability we will have
\begin{align*}
    \mathbb{E}[\delta^{\sppm}_{t+1}] &\le (1-\frac{1}{n})\frac{1}{1+\mu\eta}\mathbb{E}[\delta^{\sppm}_{t}] + \frac{1}{n}\bigl( \frac{1}{1+\mu\eta}\mathbb{E}[\delta^{\sppm}_{t}] + 2\frac{L \eta}{1+\eta\mu} \bigr) \\
    &=\frac{1}{1+\mu\eta}\mathbb{E}[\delta^{\sppm}_{t}] + \frac{2L \eta}{(1+\eta\mu)n}.\numberthis 
\end{align*}
Also, for the PPM algorithm given that $\bz_n$ denotes the different sample between datasets $S,S'$ note that
\begin{align*}
 &\begin{bmatrix}
 \bw \\
 \btheta 
 \end{bmatrix}
  = \begin{bmatrix}
 (1+\eta\mu)\bw_{\ppm} + \frac{\eta}{n}\sum_{i=1}^{n-1}\nabla_w \tilde{f}(\bw_{\ppm},\btheta_{\sppm};\bz_i) + \frac{\eta}{n}\nabla_w \tilde{f}(\bw_{\ppm},\btheta_{\sppm};\bz_n)\\
 (1+\eta\mu)\btheta_{\ppm} - \frac{\eta}{n}\sum_{i=1}^{n-1}\nabla_{\theta} \tilde{f}(\bw_{\ppm},\btheta_{\ppm};\bz_i) + \frac{\eta}{n}\nabla_{\theta} \tilde{f}(\bw_{\ppm},\btheta_{\ppm};\bz_n)
 \end{bmatrix}\\
 \Rightarrow\;\; &\frac{1}{1+\eta\mu}\begin{bmatrix}
 \bw \\
 \btheta
 \end{bmatrix} - \begin{bmatrix}
  \bw_{\sppm} \\
 \btheta_{\sppm} 
 \end{bmatrix}
  = \frac{\eta}{1+\eta\mu}\biggl(\begin{bmatrix}
 \frac{1}{n}\sum_{i=1}^{n-1}\nabla_w \tilde{f}(\bw_{\ppm},\btheta_{\ppm};\bz_i) \\
 -\frac{1}{n}\sum_{i=1}^{n-1}\nabla_{\theta} \tilde{f}(\bw_{\ppm},\btheta_{\ppm};\bz_i) 
 \end{bmatrix} \\
 &\qquad \qquad + \begin{bmatrix}
 \frac{1}{n}\nabla_w \tilde{f}(\bw_{\ppm},\btheta_{\ppm};\bz_n) \\
 -\frac{1}{n}\nabla_{\theta} \tilde{f}(\bw_{\ppm},\btheta_{\ppm};\bz_n) 
 \end{bmatrix} \biggr)
 .\numberthis
\end{align*}
In the above, the last line shows the sum of the updates for shared samples between the two datasets, i.e., $\bz_1,\ldots,\bz_{n-1}$, and the different sample $\bz_n$. Therefore, Lemma \ref{Lemma: Growth} together with Lemma \ref{Lemma: Expansive Minimax} implies that
\begin{align*}
    \delta^{\ppm}_{t+1} &\le (1-\frac{1}{n}+\frac{1}{n})\frac{1}{1+\mu\eta}\delta^{\ppm}_{t} + \frac{2L \eta}{n(1+\eta\mu)} \\
    &=\frac{1}{1+\mu\eta}\delta^{\ppm}_{t} + \frac{2L \eta}{(1+\eta\mu)n}, \numberthis
\end{align*}
which proves the same growth rule shown for SPPM also applies to the PPM algorithm. Since $\delta_0=0$, the above discussion implies the following for SPPM updates:
\begin{align*}
    \mathbb{E}[\delta^{\sppm}_{t}] &\le \frac{2L \eta}{(1+\eta\mu)n}\sum_{i=0}^t \bigl(\frac{1}{1+\mu\eta}\bigr)^i  \\
    &\le \frac{2L \eta}{(1+\eta\mu)n}\sum_{i=0}^\infty \bigl(\frac{1}{1+\mu\eta}\bigr)^i \\
    &= \frac{2L \eta}{(1+\eta\mu)n (1-1/(1+\mu\eta))} \\
    &= \frac{2L}{n\mu}.\numberthis
\end{align*}
Since $\Vert \bw_t-\bw'_t\Vert \le \delta_t$ and for every $\btheta,\bz$ $\tilde{f}(\bw,\btheta;\bz)$ is $L_w$-Lipschitz in $\bw$, the SPPM algorithm applied for $T$ iterations will be $(2LL_w/n\mu)$-uniformly stable in minimization. Therefore, the theorem's result is a corollary of Theorem 1. We can prove the result for the PPM algorithm by repeating the same steps we did for SPPM, as the two algorithms were shown to share the same growth rule.

For GDmax and PPmax algorithms, note that $\tilde{f}_{\max}(\bw;\bz):=\max_{\btheta} \tilde{f}(\bw,\btheta;\bz)-\frac{\mu}{2}\Vert\btheta \Vert^2$ will be convex and $L_w$-Lipschitz in $\bw$. Therefore, summing this function with $\frac{\mu}{2}\Vert\bw \Vert^2$ will be $\mu$-strongly convex. Since GDmax and SGDmax apply gradient descent to the maximized function, the theorem's result for GDmax and SGDmax follows from Theorem 3.9 in \citep{hardt2016train}. For SPPmax, we note that similar to Lemma \ref{Lemma: Expansive Minimax} it can be seen that the proximal point updates will be $1/(1+\mu\eta)$-expansive. Moreover for the update $\bw_{\sppmax}=G_{\sppmax}(\bw)$, we will have
\begin{align*}
 &
 \bw
  = 
 (1+\eta\mu)\bw_{\sppmax} + \eta\nabla_w \tilde{f}_{\max}(\bw_{\sppmax};\bz) \\
 \Rightarrow\;\; &\frac{1}{1+\eta\mu}
 \bw  - 
  \bw_{\sppmax} 
  = \frac{\eta}{1+\eta\mu}
 \nabla_w \tilde{f}_{\max}(\bw_{\sppmax};\bz).\numberthis
\end{align*}
As a result of Lemma 2.5 in \citep{hardt2016train}, defining $\delta^{\sppmax}_{t}=\Vert\bw_t-\bw'_t \Vert$ for datatsets $S,S'$ we will have:
\begin{align*}
    \mathbb{E}[\delta^{\sppmax}_{t+1}] &\le (1-\frac{1}{n})\frac{1}{1+\mu\eta}\mathbb{E}[\delta^{\sppmax}_{t}] + \frac{1}{n}\bigl( \frac{1}{1+\mu\eta}\mathbb{E}[\delta^{\sppmax}_{t}] + 2\frac{L_w \eta}{1+\eta\mu} \bigr) \\
    &=\frac{1}{1+\mu\eta}\mathbb{E}[\delta^{\sppmax}_{t}] + \frac{2L_w \eta}{(1+\eta\mu)n}.\numberthis 
\end{align*}
Furthermore for PPmax, we will have the following for $\bw_{\ppm}=G_{\ppm}(\bw)$ when applied to the two datasets different in only the $\bz_n$ sample:
\begin{align*}
 &
 \bw
  = 
 (1+\eta\mu)\bw_{\ppmax} + \frac{\eta}{n}\sum_{i=1}^{n-1}\nabla_w \tilde{f}_{\max}(\bw_{\ppmax};\bz_i) + \frac{\eta}{n}\nabla_w \tilde{f}_{\max}(\bw_{\ppmax};\bz_n) \\
 \Rightarrow\;\; &\frac{1}{1+\eta\mu}
 \bw  - 
  \bw_{\ppm} 
  = \frac{\eta}{1+\eta\mu} \bigl(
 \frac{1}{n}\sum_{i=1}^{n-1}\nabla_w \tilde{f}_{\max}(\bw_{\ppmax};\bz_i) \bigr) \\
 &\qquad \qquad \qquad \qquad \qquad + \frac{\eta}{1+\eta\mu} \bigl(
 \frac{1}{n}\nabla_w \tilde{f}_{\max}(\bw_{\ppmax};\bz_n) \bigr).\numberthis
\end{align*}
Applying Lemma 2.5 from \citep{hardt2016train} and defining $\delta^{\ppmax}_{t}=\Vert\bw_t-\bw'_t \Vert$ for datatsets $S,S'$ we will have:
\begin{align*}
    \delta^{\ppmax}_{t+1} &\le (1-\frac{1}{n}+\frac{1}{n})\frac{1}{1+\mu\eta}\delta^{\ppmax}_{t} + \frac{1}{n}\bigl( 2\frac{L_w \eta}{1+\eta\mu} \bigr) \\
    &=\frac{1}{1+\mu\eta}\delta^{\ppmax}_{t} + \frac{2L_w \eta}{(1+\eta\mu)n}.\numberthis 
\end{align*}

Note that $\delta^{\ppmax}_0=\delta^{\sppmax}_0=0$ which implies that:
\begin{align*}
    \mathbb{E}[\delta^{\sppmax}_{t}] &\le \frac{2L_w \eta}{(1+\eta\mu)n}\sum_{i=0}^t \bigl(\frac{1}{1+\mu\eta}\bigr)^i  \\
    &\le \frac{2L_w \eta}{(1+\eta\mu)n}\sum_{i=0}^\infty \bigl(\frac{1}{1+\mu\eta}\bigr)^i \\
    &= \frac{2L_w \eta}{(1+\eta\mu)n (1-1/(1+\mu\eta))} \\
    &= \frac{2L_w}{n\mu}.\numberthis
\end{align*}
Therefore, the SPPMax algorithm applied for $T$ iterations will be $(2L^2_w/n\mu)$-uniformly stable according to \citep{hardt2016train}'s Definition 2.1. The result is hence a consequence of Theorem 2.2 in \citep{hardt2016train}. We can prove the result for the PPmax algorithm by repeating the same steps.
\end{proof}

\subsection{Proof of Remark 1
}
\begin{remark*}
Consider a convex concave minimax objective $f(\cdot,\cdot;\bz)$ satisfying Assumptions 1 and 2. Given constant stepsizes $\alpha_w=\alpha_{\theta}=\alpha$, the GDA's generalization risk over $T$ iterations will be bounded as: 
\begin{equation}
   \epsilon_{\gen}(\gda)\le O\bigl(\frac{\alpha LL_w(1+\alpha^2\ell^2)^{T/2}}{n}\bigr). 
\end{equation}
In particular, the bound's exponential dependence on $T$ is tight for the GDA's generalization risk in the special case of $f(\bw,\btheta;\bz)=\bw^{\top}(\bz-\btheta)$.
\end{remark*}
\begin{proof}
As shown in Lemma \ref{Lemma: Expansive Minimax}, the GDA's update will be $\sqrt{1+\alpha^2\ell^2}$-expansive in this case. As a result, in learning over two datasets $S,S'$ which are different in only one sample, Lemma \ref{Lemma: Growth} shows the following growth rule for $\delta_{t}=\sqrt{\Vert \bw_t-\bw'_t\Vert_2^2+\Vert \btheta_t-\btheta'_t\Vert^2_2}$:
\begin{align*}
    \delta_{t+1} &\le (\frac{n-1}{n}+\frac{1}{n})\sqrt{1+\alpha^2\ell^2}\delta_t + \frac{2\alpha L}{n} \\
    &=\sqrt{1+\alpha^2\ell^2}\delta_t + \frac{2\alpha L}{n}.\numberthis
\end{align*}
Considering that $\delta_0=0$, we get the following exponentially growing bound in $T$ for $\delta_T$:
\begin{align*}
   \delta_T &\le \sum_{t=1}^T \bigl(1+\alpha^2\ell^2\bigr)^{t/2}\frac{2\alpha L}{n} \\
   &= \frac{2\alpha L}{n} \frac{\bigl(1+\alpha^2\ell^2\bigr)^{(T+1)/2} -1 }{\sqrt{1+\alpha^2\ell^2}-1} \\
   & = O\bigl( \frac{\alpha L\bigl(1+\alpha^2\ell^2\bigr)^{T/2}}{n} \bigr),\numberthis
\end{align*}
which considering that $f(\bw,\btheta;\bz)$ is $L_w$-Lipschitz in $\bw$ together with Theorem 1 shows that
\begin{equation}
   \epsilon_{\gen}(GDA)\le O\bigl(\frac{\alpha LL_w(1+\alpha^2\ell^2)^{T/2}}{n}\bigr). 
\end{equation}
Also, note that for the special convex-concave case $f(\bw,\btheta;\bz)=\bw^T(\bz-\btheta)$ given that $\bar{\bz}=\frac{1}{n}\sum_{i=1}^n \bz_i$ the GDA's update rule will satisfy the following
\begin{align*}
   &\begin{bmatrix} 
   \bw_{t+1}   \\
   \btheta_{t+1} - \bar{\bz}
   \end{bmatrix} = \begin{bmatrix} 
  \rmI & \alpha \rmI  \\
    -\alpha \rmI & \rmI 
   \end{bmatrix} \begin{bmatrix} 
   \bw_{t}  \\
   \btheta_{t} - \bar{\bz} 
   \end{bmatrix}, \\
   \Rightarrow \quad & \begin{bmatrix} 
   \bw_{t+1}  \\
   \btheta_{t+1}
   \end{bmatrix} = \begin{bmatrix} 
  \rmI & \alpha \rmI  \\
    -\alpha \rmI & \rmI 
   \end{bmatrix} \begin{bmatrix} 
   \bw_{t}  \\
   \btheta_{t}
   \end{bmatrix} - \begin{bmatrix} 
   \alpha \bar{\bz}  \\
   \mathbf{0}
   \end{bmatrix}.\numberthis
\end{align*}
As a result, for the updates on the two datasets $S,S'$ with size $n$ differing in only the sample $\bz_n$ we have:
\begin{align*}
   \begin{bmatrix} 
   \bw_{t+1} - \bw'_{t+1} \\
   \btheta_{t+1} - \btheta'_{t+1}
   \end{bmatrix} = \begin{bmatrix} 
  \rmI & \alpha \rmI  \\
    -\alpha \rmI & \rmI 
   \end{bmatrix} \begin{bmatrix} 
   \bw_{t} - \bw'_{t} \\
   \btheta_{t} - \btheta'_{t}
   \end{bmatrix} + \begin{bmatrix} 
   \frac{\alpha}{n}(\bz'_n-\bz_n)  \\
   \mathbf{0}
   \end{bmatrix}.\numberthis
\end{align*}
Hence, knowing that $\bw_{0} = \bw'_{0},\, \btheta_{0} = \btheta'_{0}$ we have
\begin{align*}
   \begin{bmatrix} 
   \bw_{T} - \bw'_{T} \\
   \btheta_{T} - \btheta'_{T}
   \end{bmatrix} = \begin{bmatrix} 
  \rmI & \alpha \rmI  \\
    -\alpha \rmI & \rmI 
   \end{bmatrix}^T \begin{bmatrix} 
   \frac{\alpha}{n}(\bz'_n-\bz_n)   \\
   \mathbf{0}
   \end{bmatrix}.\numberthis
\end{align*}
Since the matrix $\begin{bmatrix} 
   1 & \alpha \\
   -\alpha & 1
   \end{bmatrix}$ has the conjugate complex eigenvalues $1\pm \alpha\mathrm{i}$, we will have 
  \begin{align*}
   \big\Vert\begin{bmatrix} 
   \bw_{T} - \bw'_{T} \\
   \btheta_{T} - \btheta'_{T}
   \end{bmatrix}\big\Vert_2 = \bigl(\sqrt{1+\alpha^2}\bigr)^T \big\Vert\begin{bmatrix} 
   \frac{\alpha}{n}(\bz'_n-\bz_n)  \\
   \mathbf{0}
   \end{bmatrix} \big\Vert_2 = \frac{\alpha \bigl(\sqrt{1+\alpha^2}\bigr)^T}{n} \Vert \bz'_n-\bz_n\Vert_2\numberthis
\end{align*} 
As a consequence of the conjugate eigenvalues and the resulting iterative rotations in the complex space, the above equality shows that as long as $\alpha\neq 0$ for any constant $0<C<1$ there will exist arbitrarily large $T$ values such that
\begin{align*}
   \big\Vert
   \bw_{T} - \bw'_{T} \big\Vert_2 \ge \frac{C\alpha \bigl(\sqrt{1+\alpha^2}\bigr)^T}{n} \Vert \bz'_n-\bz_n\Vert_2.\numberthis
\end{align*} 
Equivalently, we have
\begin{equation}
    \Vert \bw_{T} - \bw'_{T} \Vert_2 =\Omega_T \bigl(\frac{\alpha (1+\alpha^2\bigr)^{T/2}}{n}\Vert\bz'_n-\bz_n\Vert_2\bigr),
\end{equation}
which proves the exponential dependence of the expected generalization risk on $T$ and completes the proof.
\end{proof}

\subsection{Proof of Theorem 3
}
Here we prove the following generalized version of Theorem 3 in the text for general time-varying stepsize values.
\begin{thm*}
Consider a convex-concave minimax learning objective $f(\cdot,\cdot;\bz)$ satisfying Assumptions 1 and 2 for every $\bz$. Then, stochastic PPM with stepsizes $\eta_{t}$ at iteration $t$ over $T$ iterations will satisfy 
\begin{equation}
    \epsilon_{\gen}(\ppm),\epsilon_{\gen}(\sppm)\le \frac{2LL_w}{n}\sum_{t=1}^T\eta_{t},\quad \epsilon_{\gen}(\ppmax),\epsilon_{\gen}(\sppmax)\le \frac{2L_w^2}{n}\sum_{t=1}^T\eta_{t}.
\end{equation}
\end{thm*}
\begin{proof}
Consider two datasets $S,S'$ with size $n$ which have only one different sample. As a result of Lemma \ref{Lemma: Expansive Minimax}, the proximal point updates will be $1$-expansive. Therefore, according to Lemma \ref{Lemma: Growth} and the Law of total probability, defining $\delta^{\sppm}_t= \sqrt{\Vert \bw_t-\bw'_t \Vert^2+\Vert \btheta_t-\btheta'_t\Vert^2}$ we will have
\begin{align*}
    \bbE[\delta^{\sppm}_{t+1}] &\le (1-\frac{1}{n})\bbE[\delta^{\sppm}_{t}]+ \frac{1}{n}\bigl(\bbE[\delta^{\sppm}_{t}]+ 2\eta_{t} L\bigr) \\
    &= \bbE[\delta^{\sppm}_{t}] + \frac{2\eta_{t} L}{n}\numberthis
\end{align*}
Given that $\delta^{\sppm}_{0}=0$, we reach the following inequality for every $T$
\begin{align*}
    \bbE[\delta^{\sppm}_{T}] \le \frac{2 L}{n}\sum_{t=1}^T\eta_{t}  .\numberthis
\end{align*}
Note that for every $\btheta,\bz$, $f(\bw,\btheta;\bz)$ is $L_w$-Lipschitz, which with the above inequality implies that SPPM will be uniformly-stable in minimization with the following degree
\begin{align*}
     \frac{2LL_w}{n}\sum_{t=1}^T\eta_{t}.\numberthis
\end{align*}
The theorem's result for SPPM then becomes a consequence of Theorem 1. Furthermore, regarding the PPM algorithm applying Lemma \ref{Lemma: Growth} and Lemma \ref{Lemma: Expansive Minimax} implies that
\begin{align*}
    \delta^{\ppm}_{t+1} &\le (1-\frac{1}{n}+\frac{1}{n})\delta^{\ppm}_{t}+ \frac{2\eta_{t} L}{n} \\
    &= \delta^{\ppm}_{t} + \frac{2\eta_{t} L}{n}.\numberthis
\end{align*}
The above equation holds because the update rule of PPM can be written in the following way where $\bz_n$ denotes the only different sample between the two datasets,
\begin{align*}
    &\begin{bmatrix}
 \bw \\
 \btheta
 \end{bmatrix} - \begin{bmatrix}
  \bw_{\ppm} \\
 \btheta_{\ppm} 
 \end{bmatrix}
  = \eta\biggl(\begin{bmatrix}
 \frac{1}{n}\sum_{i=1}^{n-1}\nabla_w f(\bw_{\ppm},\btheta_{\ppm};\bz_i) \\
 -\frac{1}{n}\sum_{i=1}^{n-1}\nabla_{\theta} f(\bw_{\ppm},\btheta_{\ppm};\bz_i) 
 \end{bmatrix} + \begin{bmatrix}
 \frac{1}{n}\nabla_w f(\bw_{\ppm},\btheta_{\ppm};\bz_n) \\
 -\frac{1}{n}\nabla_{\theta} f(\bw_{\ppm},\btheta_{\ppm};\bz_n) 
 \end{bmatrix} \biggr)
 .\numberthis
\end{align*}
Since $\delta^{\ppm}_{0}=0$, at iteration $T$ we have
\begin{align*}
    \delta^{\ppm}_{T} \le \frac{2 L}{n}\sum_{t=1}^T\eta_{t}  .\numberthis
\end{align*}
As a result, we can repeat the last step of our proof for the case of SPPM to complete the proof for the PPM case. For the PPmax and SPPmax algorithms, note that $f_{\max}(\bw;\bz):=\max_{\btheta} f(\bw,\btheta;\bz)$ will be convex and $L_w$-Lipschitz in $\bw$. The result is therefore a corollary of Theorem 3.8 and Lemma 4.6 in \citep{hardt2016train}.
\end{proof}

\subsection{Proof of Theorem 4}
\begin{thm*}
Given a differentiable minimax objective $f(\bw,\btheta;\bz)$ the average iterate updates $\bar{\bw}^{(T)}:=\frac{1}{T}\sum_{t=1}^T \bw^{(t)},\, \bar{\btheta}^{(T)}:=\frac{1}{T}\sum_{t=1}^T \btheta^{(t)}$ of SPPM and SPPmax with setpsize parameter $\eta$ will satisfy the following optimality gaps for a saddle solution $[\bw^*_S,\btheta^*_S]$ of the empirical risk for dataset $S$:
\begin{align*}
    &\sppm:\: \mathbb{E}\bigl[\, R_S({\bar{\bw}}^{(T)}) \,\bigr]- R_S(\bw^*_S)\le \frac{\Vert\bw^{(0)}-\bw^*_S\Vert^2+\Vert\btheta^{(0)}-\btheta^*_S\Vert^2}{2\eta T}, \\
    &\sppmax:\: \mathbb{E}\bigl[ R_S(\bar{\bw}^{(T)})  \bigr]-R_S(\bw^*_S)\le \frac{\Vert\bw^{(0)}-\bw^*_S\Vert^2}{2\eta T}.\numberthis
\end{align*}
\end{thm*}
\begin{proof}
Note that for any proximal operator $F_k$ such that $\bv_{k+1} = \bv_{k}-\eta F_k(\bv_{k+1}) $ we will have the following for every $\bv$:
\begin{align*}
    &\frac{1}{2\eta}\Vert\bv_k -\bv \Vert^2 - \frac{1}{2\eta}\Vert\bv_{k+1} -\bv \Vert^2 - \frac{1}{2\eta}\Vert\bv_{k+1} -\bv_k \Vert^2 \\
    =\, &-\frac{1}{\eta}\bigl(\Vert \bv_{k+1}\Vert^2 - \bv_k^T\bv_{k+1} - \bv^T\bv_{k+1} + \bv^T\bv_{k}\bigr) \\
    =\, & -\frac{1}{\eta}(\bv_{k+1}-\bv_k)^T(\bv_{k+1}-\bv) \\
    =\, &  F_k(\bv_{k+1})^T(\bv_{k+1}-\bv).\numberthis
\end{align*}
As a result, we have
\begin{align*}
    \frac{1}{T}\sum_{k=0}^T F_k(\bv_{k})^T(\bv_{k}-\bv) &= \frac{1}{2\eta T}\Vert\bv_0 -\bv \Vert^2 - \frac{1}{2\eta T}\Vert\bv_{T} -\bv \Vert^2 - \frac{1}{T}\sum_{k=0}^{T}\Vert\bv_{k} -\bv_{k-1} \Vert^2 \\
    &\le \frac{1}{2\eta T}\Vert\bv_0 -\bv \Vert^2.\numberthis
\end{align*}
Given that every $F_k$ is a stochastic proximal rule for a uniformly random training sample, the law of iterated expectation conditioned to random update $\bv_t$ at iteration $t$ implies that
\begin{align*}
    \bbE\biggl[ \frac{1}{T}\sum_{k=0}^T F_k(\bv_{k})^T(\bv_{k}-\bv) \biggr] &= \frac{1}{T}\sum_{k=0}^T \bbE\biggl[F_k(\bv_{k})^T(\bv_{k}-\bv)\biggr] \\
    &= \frac{1}{T}\sum_{k=0}^T \bbE\biggl[\bbE\big[F_k(\bv_{k})^T(\bv_{k}-\bv)\big\vert \bv_{k}  \bigr]\biggr] \\
    &= \frac{1}{T}\sum_{k=0}^T \bbE\biggl[\bbE\big[F_k(\bv_{k})\big\vert \bv_{k}  \bigr]^T(\bv_{k}-\bv)\biggr] \\
    &= \frac{1}{T}\sum_{k=0}^T \bbE\biggl[\bbE\big[\bar{F}(\bv_{k})\big\vert \bv_{k}  \bigr]^T(\bv_{k}-\bv)\biggr] \\
    &= \frac{1}{T}\sum_{k=0}^T \bbE\biggl[\bar{F}(\bv_{k})^T(\bv_{k}-\bv)\biggr] \\
    &= \bbE\biggl[\frac{1}{T}\sum_{k=0}^T \bar{F}(\bv_{k})^T(\bv_{k}-\bv)\biggr]\numberthis
\end{align*}
where $\bar{F}$ denotes the gradient update for the averaged loss over the training samples. Therefore, we have
\begin{equation}
    \bbE\biggl[\frac{1}{T}\sum_{k=0}^T \bar{F}(\bv_{k})^T(\bv_{k}-\bv)\biggr] \le \frac{1}{2\eta T}\Vert\bv_0 -\bv \Vert^2.
\end{equation}
Considering the optimal saddle solution $\bv=[\bw^*_S,\btheta^*_S]$ for the SPPM algorithm, combining the above result with Lemma 2 in \citep{mokhtari2019convergence} proves the theorem's result on the convergence of SPPM's average iterates. For the convergence result on SPPmax updates, note that given a convex function $f$ and its gradient $F$ and minimizer $\bv^*$ we have
\begin{equation}
    f\bigl(\frac{1}{T}\sum_{t=1}^T \bv^{(t)}\bigr) - f(\bv^*)\le \frac{1}{T}\sum_{t=1}^T\bigl[ f(\bv^{(t)}) - f(\bv^*) \bigr] \le \frac{1}{T}\sum_{t=1}^T F(\bv^{(t)})^T(\bv^{(t)}-\bv^*). 
\end{equation}
The above equation together with the property shown for the stochastic updates of SPPmax completes the theorem's proof.
\end{proof}

\subsection{Proof of Corollary 1
}
\begin{cor*}
Consider a convex concave minimax objective which we optimize via PPM and PPmax with setpsize parameter $\eta$. Then, given that $\Vert\bw^{(0)}-\bw^*_S\Vert^2+\Vert\btheta^{(0)}-\btheta_S^*\Vert^2\le D^2$ for PPM and $\Vert\bw^{(0)}-\bw_S^*\Vert\le D$ for PPmax holds with probability $1$, it will take $T_{\sppm}=\sqrt{\frac{nD^2}{2\eta^2LL_w}}$ and $T_{\sppmax}=\sqrt{\frac{nD^2}{2\eta^2L^2_w }}$ iterations for the average iterates to achieve the following bounded excess risks where $\bw^*$ denotes the optimal learner minimizing the true risk $R(\bw)$:
\begin{align*}
    &\ppm,\sppm:\: \mathbb{E}\bigl[ R(\bar{\bw}^{(T_{\sppm})}) \bigr]- R(\bw^*)\le \sqrt{\frac{2D^2LL_w}{n}}, \\
    &\ppmax,\sppmax:\: \mathbb{E}\bigl[ R(\bar{\bw}^{(T_{\sppmax})})\bigr]- R(\bw^*) \le \sqrt{\frac{2D^2L^2_w}{n}}.\numberthis
\end{align*}
\end{cor*}
\begin{proof}
First, we show that using a constant stepsize parameter $\eta$ the average iterates reach $1/2$ of the generalization bound for the final iterates in Theorem 3. For the average iterates $(\bar{\bw}_t,\bar{\btheta}_t)$ and $(\bar{\bw}'_t,\bar{\btheta}'_t)$ we have the following application of Jensen's inequality on the convex norm function for the difference of average iterates $\bar{\delta}_t=\sqrt{\Vert\bar{\bw}^{(t)}-\bar{\bw}'^{(t)} \Vert^2+\Vert\bar{\btheta}^{(t)}-\bar{\btheta}'^{(t)}\Vert^2}$
\begin{align*}
   \bar{\delta}_t :&= \sqrt{\Vert\bar{\bw}^{(t)}-\bar{\bw}'^{(t)} \Vert^2+\Vert\bar{\btheta}^{(t)}-\bar{\btheta}'^{(t)}\Vert^2} \\
   &\le \frac{1}{t}\sum_{k=0}^{t-1}\sqrt{\Vert{\bw}_k-{\bw}'_k \Vert^2+\Vert\btheta_k-{\btheta}'_k\Vert^2 } \\
   &= \frac{1}{t}\sum_{k=0}^{t-1} {\delta}_t.\numberthis
\end{align*}
Similarly, one can show that $\bar{\delta}_{w,t}\le \frac{1}{t}\sum_{k=1}^t {\delta}_{w,t}$. Therefore, knowing that $\mathbb{E}[\delta_t]\le \frac{2LL_wt\eta}{n}$ implies that 
\begin{align*}
    \bbE[\bar{\delta}_t] \le \frac{1}{t}\sum_{k=1}^t \bbE[{\delta}_t] \le \frac{1}{t}\sum_{k=0}^{t-1} \frac{2LL_wk\eta}{n} 
    \le \frac{LL_wt\eta}{n}.\numberthis
\end{align*}
Hence, at the $T$th average iterate of PPM and SPPM we will have
\begin{align*}
    \mathbb{E}_A[R(\bar{\bw}^{(T)})] - R_S[\bar{\bw}^{(T)}] \le \frac{LL_wT\eta}{n}\numberthis
\end{align*}
which together with \citep{mokhtari2019convergence}'s Theorem 1 for the PPM and and our generalization of that theorem to stochastic PPM in Theorem 4 shows that
\begin{align*}
    \mathbb{E}_{A,S}[R(\bar{\bw}^{(T)})] - \bbE_S[R_S[\bar{\bw}^{(T)}]]   \le \frac{LL_w\eta T}{n} + \frac{D^2}{2\eta T}.\numberthis
\end{align*}
Note that $\bbE_S[R_S(\mathbf{w}_S)]\le\bbE_S[R_S(\mathbf{w}^*)]= R(\bw^*)$, indicating that
\begin{equation}
    \mathbb{E}_{A,S}[R(\bar{\bw}^{(T)})] -R(\bw^*)   \le \frac{LL_w\eta T}{n} + \frac{D^2}{2\eta T}.
\end{equation}
The above upper-bound will be minimized when $\eta T =\sqrt{\frac{n D^2}{2LL_w}}$ and the optimized excess risk upper-bound for PPM and SPPM will be
\begin{equation}
    \mathbb{E}_{A,S}[R(\bar{\bw}^{(T)})] -R(\bw^*)   \le \sqrt{\frac{ 2LL_wD^2}{n}} .
\end{equation}
Similarly, it can be seen that for PPmax and SPPmax the optimal bound will be achieved at $\eta T =\sqrt{\frac{n D^2}{2L^2_w}}$ which suggests the following excess risk bound:
\begin{equation}
    \mathbb{E}_{A,S}[R(\bar{\bw}^{(T)})] -R(\bw^*)   \le \sqrt{\frac{ 2L^2_wD^2}{n}} .
\end{equation}
The proof is therefore complete.
\end{proof}

\subsection{Proof of Theorem 5
}
\begin{thm*}
Let learning objective $f(\bw,\btheta;\bz)$ be non-convex $\mu$-strongly-concave and satisfy Assumptions 1 and 2. Also, we assume that $f_{\max}(\bw;\bz):=\max_{\btheta\in \Theta}f(\bw,\btheta;\bz)$ is bounded as $0\le f_{\max}(\bw;\bz)\le 1$ for every $\bw,\bz$. Then, defining $\kappa:={\ell}/{\mu}$ we have
\begin{enumerate}[wide,labelwidth=!,labelindent=0pt,topsep=1pt,itemsep=0pt]
    \item The SGDA algorithm with vanishing stepsizes $\alpha_{w,t}={c}/{ t},\, \alpha_{\theta,t}={cr^2}/{t}$ for constants $c>0,1\le r\le \kappa$ satisfies the following bound over $T$ iterations: 
    \begin{equation}
        \epsilon_{\gen}(\sgda)\le \frac{1+\frac{1}{(r+1)c\ell }}{n}(12(r+1)cLL_w)^{\frac{1}{(r+1)c\ell +1}}T^{\frac{(r+1)c\ell }{(r+1)c\ell +1}}.
    \end{equation}
    \item The SGDmax algorithm with vanishing stepsize $\alpha_{w,t}={c}/{t}$ for constant $c>0$  satisfies the following bound over $T$ iterations: 
    \begin{equation}
        \epsilon_{\gen}(\sgdmax)\le \frac{1+\frac{2}{(\kappa+2)\ell c}}{n-1}\bigl(2cL_w^2\bigr)^{\frac{2}{(\kappa+2)\ell c+2}}T^{\frac{(\kappa+2)\ell c}{(\kappa+2)\ell c+2}}.
    \end{equation}
\end{enumerate}
\end{thm*}
\begin{proof}
We start by proving the following lemmas.
\begin{lemma}\label{Lemma: non-convex bounded}
Let $f(\bw,\btheta;\bz)$ be $L_w$-Lipschitz in $\bw$ and assume that $f_{\max}(\bw;\bz):=\max_{\btheta}f(bw,\btheta;\bz$ is bounded $0\le f_{\max}(\bw;\bz)\le 1$. Then, in applying SGDA for learning over two datasets $S,S'$ which differ in only one sample the updated variables $\bw_t,\bw'_t$ will satisfy the following inequality for every $t_0\in\{1,\ldots,n\}$ where $\delta_t:=\sqrt{\Vert \bw_t-\bw'_t\Vert^2+\Vert\btheta-\btheta'_t\Vert^2}$:
\begin{equation}
    \forall \bz:\quad \mathbb{E}\bigl[\vert f_{\max}(\bw_t;\bz)-f_{\max}(\bw'_t;\bz)\vert \bigr] \le \frac{t_0}{n} + L_w\mathbb{E}[\delta_t|\delta_{t_0}=0].
\end{equation}
\end{lemma}
\begin{proof}
Define the event $E_{t_0}=\mathbb{I}(\delta_{t_0}=0)$ as the indicator of the outcome $\delta_{t_0}=0$. Then, due to the law of total probability
\begin{align*}
    \mathbb{E}\bigl[\vert f_{\max}(\bw_t;\bz)-f_{\max}(\bw'_t;\bz)\vert \bigr]
    =  &\Pr(E_{t_0})\mathbb{E}\bigl[\vert f_{\max}(\bw_t;\bz)-f_{\max}(\bw'_t;\bz)\vert\, \big\vert E_{t_0} \bigr] \\
    &\quad + \Pr(E^{c}_{t_0})\mathbb{E}\bigl[\vert f_{\max}(\bw_t;\bz)-f_{\max}(\bw'_t;\bz)\vert\, \big\vert E^{c}_{t_0} \bigr] \\
    \stackrel{(a)}{\le} & \mathbb{E}\bigl[\vert f_{\max}(\bw_t;\bz)-f_{\max}(\bw'_t;\bz)\vert\, \big\vert E_{t_0} \bigr] + \Pr(E^{c}_{t_0}) \\
    \stackrel{(b)}{\le} & L_w\mathbb{E}\bigl[\Vert \bw_t-\bw'_t\Vert\, \big\vert E_{t_0} \bigr] + \Pr(E^{c}_{t_0}) \\
    \stackrel{(c)}{\le} & L_w\mathbb{E}\bigl[\delta_t\, \big\vert \delta_{t_0}=0 \bigr] + \frac{t_0}{n}.\numberthis 
\end{align*}
In the above equations, (a) follows from the boundedness assumption on $f_{\max}$. (b) is the consequence of $L_w$-Lipschitzness of $f$ which also transfers to $f_{\max}$. Finally, (c) holds because $\Vert \bw_t - \bw'_t\Vert\le \delta_t$ according to the definition. Then, using the union bound on the outcome $I=I_t$ where $I$ is the index of different samples in $S,S'$ and $I_t$ is the index of sample used by SGDA at iteration $t$ we obtain that
\begin{equation}
    \Pr(E^{c}_{t_0}) = \Pr(\delta_{t_0}>0) \le \sum_{i=1}^{t_0}\Pr(I=I_i) =\frac{t_0}{n}.
\end{equation}
The lemma's proof is therefore complete.
\end{proof}
In order to prove the theorem for SGDA updates, we provide an extension of Lemma \ref{Lemma: Expansive Minimax} for non-convex concave minimax objectives.
\begin{lemma}\label{Lemma: Expansive Non-convex Concave Minimax}
Consider a non-convex $\mu$-strongly concace objective $f(\bw,\btheta)$ satisfying Assumption 2. Then, for every two pairs $(\bw,\btheta),\, (\bw',\btheta')$ the GDA updates $[\bw_{\gda},\btheta_{\gda}]=G_{\gda}([\bw,\btheta])$, $[\bw'_{\gda},\btheta'_{\gda}]=G_{\gda}([\bw',\btheta'])$ with stepsizes $\alpha_{w},\alpha_{\theta}\le\frac{1}{\ell}$ will satisfy the following expansivity equation:
\begin{equation}
    \begin{bmatrix}
    \Vert \bw_{\gda} - \bw'_{\gda} \Vert \\
    \Vert \btheta_{\gda} - \btheta'_{\gda} \Vert
    \end{bmatrix} \le  \begin{bmatrix}
    1+\alpha_w\ell  & \alpha_w\ell \\
    \alpha_{\theta}\ell & 1-\frac{ \alpha_{\theta}\mu}{2}
    \end{bmatrix} \begin{bmatrix}
    \Vert \bw - \bw' \Vert \\
    \Vert \btheta - \btheta' \Vert
    \end{bmatrix}.
\end{equation}
\end{lemma}
\begin{proof}
Note that
\begin{align*}
     \Vert \bw_{\gda} - \bw'_{\gda} \Vert &= \Vert \bw - \alpha_w\nabla_{\bw}f(\bw,\btheta) - \bw' + \alpha_w\nabla_{\bw}f(\bw',\btheta') \Vert  \\
     &\le \Vert \bw - \alpha_w\nabla_{\bw}f(\bw,\btheta) - \bw' + \alpha_w\nabla_{\bw}f(\bw',\btheta) \Vert \\
     &\quad + \Vert \alpha_w\nabla_{\bw}f(\bw',\btheta) - \alpha_w\nabla_{\bw}f(\bw',\btheta')\Vert \\
     &\le (1+\alpha_w\ell) \Vert \bw -\bw'\Vert + \alpha_w\ell\Vert \btheta -\btheta'\Vert.\numberthis
\end{align*}
Furthermore, we have
\begin{align*}
     \Vert \btheta_{\gda} - \btheta'_{\gda} \Vert &= \Vert \btheta + \alpha_{\theta}\nabla_{\btheta}f(\bw,\btheta) - \btheta' - \alpha_{\theta}\nabla_{\btheta}f(\bw',\btheta') \Vert  \\
     &\le \Vert \btheta + \alpha_{\theta}\nabla_{\btheta}f(\bw,\btheta) - \btheta' - \alpha_{\theta}\nabla_{\btheta}f(\bw,\btheta') \Vert \\
     &\quad + \Vert \alpha_{\theta}\nabla_{\btheta}f(\bw,\btheta') - \alpha_{\theta}\nabla_{\btheta}f(\bw',\btheta')\Vert \\
     &\le \bigl(1-\frac{\alpha_{\theta}\mu}{2}\bigr)\Vert \btheta -\btheta'\Vert + \alpha_{\theta}\ell\Vert \bw -\bw'\Vert ,\numberthis
\end{align*}
where the last inequality follows from Lemma 3.7 in \citep{hardt2016train} knowing that $\mu\le \ell$. Therefore, the lemma's proof is complete.
\end{proof} 
\begin{lemma}\label{Lemma: Growth non-convex concave minimax}
Consider two sequence of updates $G_1,\ldots,G_T$ and $G'_1,\ldots,G'_T$ for minimax objective $f(\bw,\btheta)$. Define $\delta_{w,t}=\Vert \bw_t-\bw'_t \Vert$ and $\delta_{\theta,t}=\Vert \btheta_t-\btheta'_t \Vert$. Assume that $G_t$ is $\boldsymbol{\eta}$-expansive for matrix $\eta_{2\times 2}$, i.e. it satisfies the following inequality for every $[\bw_{G_t},\btheta_{G_t}]:= G_t(\bw,\btheta)$, $[\bw'_{G_t},\btheta'_{G_t}]:= G_t(\bw',\btheta')$
\begin{equation}
    \begin{bmatrix}
    \Vert \bw_{G_t} - \bw'_{G_t} \Vert \\
    \Vert \btheta_{G_t} - \btheta'_{G_t} \Vert
    \end{bmatrix} \le \boldsymbol{\eta}\begin{bmatrix}
    \Vert \bw - \bw' \Vert \\
    \Vert \btheta - \btheta' \Vert
    \end{bmatrix}.
\end{equation}
Also, suppose that for every $[\bw_{G_t},\btheta_{G_t}]:= G_t(\bw,\btheta),\,  [\bw_{G_t},\btheta_{G'_t}]:= G'_t(\bw,\btheta)$ we have
\begin{align*}
    \sup_{\bw,\btheta} \Vert\bw_{G_t}- \bw\Vert \le \sigma_w &, \quad \sup_{\bw,\btheta} \Vert\btheta_{G_t}- \btheta\Vert \le \sigma_{\theta}, \\
    \sup_{\bw,\btheta} \Vert\bw_{G'_t}- \bw\Vert \le \sigma_w &, \quad \sup_{\bw,\btheta} \Vert\btheta_{G'_t}- \btheta\Vert \le \sigma_{\theta}.\numberthis
\end{align*}
Then, we have
\begin{equation}
  \begin{bmatrix}
    \delta_{w,t+1} \\
    \delta_{\theta,t+1}
    \end{bmatrix} \le \boldsymbol{\eta} \begin{bmatrix}
    \delta_{w,t} \\
    \delta_{\theta,t}
    \end{bmatrix}  + 2\begin{bmatrix}
    \sigma_w \\
    \sigma_{\theta}
    \end{bmatrix}  .
\end{equation}
\end{lemma}
\begin{proof}
Note that
\begin{align*}
  \begin{bmatrix}
    \delta_{w,t+1} \\
    \delta_{\theta,t+1}
    \end{bmatrix} &=  \begin{bmatrix}
    \Vert G_{t,w}(\bw_t,\btheta_t) -G'_{t,w}(\bw'_t,\btheta'_t) \Vert \\
    \Vert G_{t,\theta}(\bw_t,\btheta_t) -G'_{t,\theta}(\bw'_t,\btheta'_t) \Vert
    \end{bmatrix} \\ 
    &=  \begin{bmatrix}
    \Vert G_{t,w}(\bw_t,\btheta_t) - G_{t,w}(\bw'_t,\btheta'_t) + G_{t,w}(\bw'_t,\btheta'_t) -G'_{t,w}(\bw'_t,\btheta'_t) \Vert \\
    \Vert G_{t,\theta}(\bw_t,\btheta_t) - G_{t,\theta}(\bw'_t,\btheta'_t) + G_{t,\theta}(\bw'_t,\btheta'_t) -G'_{t,\theta}(\bw'_t,\btheta'_t) \Vert
    \end{bmatrix} \\
    &=  \begin{bmatrix}
    \Vert G_{t,w}(\bw_t,\btheta_t) - G_{t,w}(\bw'_t,\btheta'_t) \Vert \\
    \Vert G_{t,\theta}(\bw_t,\btheta_t) - G_{t,\theta}(\bw'_t,\btheta'_t) \Vert
    \end{bmatrix} + \begin{bmatrix}
    \Vert  G_{t,w}(\bw'_t,\btheta'_t) -G'_{t,w}(\bw'_t,\btheta'_t) \Vert \\
    \Vert G_{t,\theta}(\bw'_t,\btheta'_t) -G'_{t,\theta}(\bw'_t,\btheta'_t) \Vert
    \end{bmatrix}  \\
    &=  \begin{bmatrix}
    \Vert G_{t,w}(\bw_t,\btheta_t) - G_{t,w}(\bw'_t,\btheta'_t) \Vert \\
    \Vert G_{t,\theta}(\bw_t,\btheta_t) - G_{t,\theta}(\bw'_t,\btheta'_t) \Vert
    \end{bmatrix} + \begin{bmatrix}
    \Vert  G_{t,w}(\bw'_t,\btheta'_t) - \bw'_t  \Vert \\ \Vert
    G_{t,\theta}(\bw'_t,\btheta'_t) - \btheta'_t  \Vert
    \end{bmatrix} \\
    &\quad + \begin{bmatrix}
    \Vert  \bw'_t -G'_{t,w}(\bw'_t,\btheta'_t) \Vert \\
    \Vert\btheta'_t -G'_{t,\theta}(\bw'_t,\btheta'_t) \Vert
    \end{bmatrix}  \\
    &\le \boldsymbol{\eta} \begin{bmatrix}
    \delta_{w,t} \\
    \delta_{\theta,t}
    \end{bmatrix}  + 2\begin{bmatrix}
    \sigma_w \\
    \sigma_{\theta}
    \end{bmatrix},  \numberthis
\end{align*}
which makes the proof complete.
\end{proof}
\begin{lemma}\label{Lemma: Smoothness of maximized objective}
Consider a non-convex $\mu$-strongly convex minimax objective $f(\bw,\btheta)$ satisfying Assumption 2 over a convex feasible set $\Theta$. Then, the maximized objective $f_{\max}(\bw):= \max_{\btheta\in\Theta} f(\bw,\btheta)$ will be $(\ell + \ell^2/2\mu)$-smooth, i.e., for every $\bw_1,\bw_2\in\mathcal{W}$ it satisfies 
\begin{equation}
    \bigl\Vert \nabla f_{\max}(\bw_2) - \nabla f_{\max}(\bw_1) \bigr\Vert_2\le \bigl(\ell + \frac{\ell^2}{2\mu}\bigr) \bigl\Vert \bw_2-\bw_1 \bigr\Vert_2.
\end{equation}
\end{lemma}
\begin{proof}
Consider two arbitrary points $\bw_1,\bw_2\in\mathcal{W}$ and define $\btheta^*(\bw_1),\btheta^*(\bw_2)$ as the optimal maximizers over $\Theta$ for $f(\bw_1,\cdot),f(\bw_2,\cdot)$, respectively. Since, $f(\bw,\cdot)$ is $\ell$-smooth and $\mu$-strongly-convex, there exists a unique solution $\btheta^*(\bw)$ for every $\bw$. Then, the $\mu$-strongly concavity implies that
\begin{equation}
    \mu\big\Vert \btheta^*(\bw_1) - \btheta^*(\bw_2)\big\Vert^2_2 \le \bigl(\btheta^*(\bw_2) - \btheta^*(\bw_1)\bigr)^T\bigl(\nabla_{\btheta}f(\bw_1,\btheta^*(\bw_1)) - \nabla_{\btheta}f(\bw_1,\btheta^*(\bw_2)) \bigr).
\end{equation}
Due to the optimality of $\btheta^*(\bw_1),\btheta^*(\bw_2)$ over the convex feasible set $\Theta$ we further have
\begin{align*}
   &\bigl(\btheta^*(\bw_2) - \btheta^*(\bw_1)\bigr)^T\bigl(\nabla_{\btheta}f(\bw_1,\btheta^*(\bw_1)) -\nabla_{\btheta}f(\bw_2,\btheta^*(\bw_2))\bigr) \\
   = \, & \bigl(\btheta^*(\bw_2) - \btheta^*(\bw_1)\bigr)^T\nabla_{\btheta}f(\bw_1,\btheta^*(\bw_1)) + \bigl(\btheta^*(\bw_1) - \btheta^*(\bw_2)\bigr)^T\nabla_{\btheta}f(\bw_2,\btheta^*(\bw_2)) \\
   \le\, & 0.   \numberthis
\end{align*}
Combining the above two equations, we obtain
\begin{align*}
   \mu\big\Vert \btheta^*(\bw_1) - \btheta^*(\bw_2)\big\Vert^2_2 &\le  \bigl(\btheta^*(\bw_2) - \btheta^*(\bw_1)\bigr)^T\bigl(\nabla_{\btheta}f(\bw_2,\btheta^*(\bw_2)) - \nabla_{\btheta}f(\bw_1,\btheta^*(\bw_2)) \bigr) \\
   &\le \ell \big\Vert \btheta^*(\bw_1) - \btheta^*(\bw_2)\big\Vert_2 \Vert \bw_2 - \bw_1\Vert_2. \numberthis
\end{align*}
The above equation results in
\begin{equation}\label{Eq: Proof of Lemma 6, smoothness of max}
    \big\Vert \btheta^*(\bw_1) - \btheta^*(\bw_2)\big\Vert_2 \le \frac{\ell}{\mu} \Vert \bw_2 - \bw_1\Vert_2.
\end{equation}
As a result, applying the Danskin's theorem for smooth objectives with a unique solution \citep{bernhard1995theorem} implies that
\begin{align*}
    \bigl\Vert \nabla f_{\max}(\bw_2) - \nabla f_{\max}(\bw_1) \bigr\Vert_2 &=  \bigl\Vert \nabla_{\bw} f(\bw_2,\btheta^*(\bw_2)) - \nabla_{\bw} f(\bw_1,\btheta^*(\bw_1)) \bigr\Vert_2 \\
    &\le \ell \sqrt{\big\Vert \bw_2 -\bw_1\big\Vert_2^2+ \big\Vert \btheta^*(\bw_1) - \btheta^*(\bw_2)\big\Vert^2_2 } \\
    &\le \ell \sqrt{ \bigl(1+(\ell/\mu)^2\bigr)\Vert \bw_2 -\bw_1\Vert_2^2 }\\
    &= \ell\sqrt{1+(\ell/\mu)^2}\Vert \bw_2 -\bw_1\Vert_2 \\
    &\le \bigl(\ell+\frac{\ell^2}{2\mu^2}\bigr)\Vert \bw_2 -\bw_1\Vert_2,\numberthis
\end{align*}
where the last line holds since $\sqrt{1+t}\le 1+t/2$ for every $t\ge -1$. The proof is hence complete. 
\end{proof}
To prove the theorem's result on SGDA note that Lemma \ref{Lemma: Expansive Non-convex Concave Minimax} suggests that the SGDA update at iteration $t$ for non-convex non-concave problems will be expansive with the following matrix:
\begin{equation}
 B_t:=\begin{bmatrix}
    1+\alpha_{w,t}\ell  & \alpha_{w,t}\ell \\
    \alpha_{\theta,t}\ell & 1-\frac{ \alpha_{\theta,t}\mu}{2}
    \end{bmatrix} = I + \alpha_{w,t}\ell  \begin{bmatrix}
    1  & 1 \\
    \frac{\alpha_{\theta,t}}{\alpha_{w,t}} & -\frac{ \mu\alpha_{\theta,t}}{\ell\alpha_{w,t}}
    \end{bmatrix} = I + \frac{c\ell}{t}  \begin{bmatrix}
    1  & 1 \\
    r^2 & -r^2/\kappa
    \end{bmatrix}.
\end{equation}
For analyzing the powers of the above matrix, we diagonalize it using its eigenvalues $\lambda_1,\lambda_2$ and corresponding eigenvectors $\boldsymbol{\nu}_1,\boldsymbol{\nu}_2$. Note that the product of the eigenvalues of $\begin{bmatrix}
    1  & 1 \\
    r^2 & -r^2/\kappa
    \end{bmatrix}$, i.e. the matrix's determinant, is negative and hence the matrix has two different real eigenvalues with opposite signs. This implies that the matrix is diagonlizable and so is a linear combination of the matrix with the identity matrix. As a result, given the invertible matrix $\boldsymbol{\nu}=[\boldsymbol{\nu}_1,\boldsymbol{\nu}_2]$ we have
\begin{equation}
    B_t=\begin{bmatrix}
    1+\alpha_{w,t}\ell  & \alpha_{w,t}\ell \\
    \alpha_{\theta,t}\ell & 1-\frac{ \alpha_{\theta,t}\mu}{2}
    \end{bmatrix} = \boldsymbol{\nu}^{-1} \begin{bmatrix}
    1+\frac{c\ell\lambda_1}{t}  & 0 \\
    0 & 1+\frac{c\ell\lambda_2}{t}
    \end{bmatrix}\boldsymbol{\nu}.
\end{equation}
Also, notice that we have the following closed-form solution for $\lambda_1,\lambda_2$:
\begin{equation}
    \lambda_1 = \frac{\kappa-r^2+\sqrt{4\kappa^2r^2+(\kappa+r^2)^2}}{2\kappa},\quad \lambda_2 = \frac{\kappa-r^2-\sqrt{4\kappa^2r^2+(\kappa+r^2)^2}}{2\kappa}.
\end{equation}
Therefore, since we assume $1\le r\le \kappa$,
\begin{equation}
    \max\{\lambda_1,\lambda_2\} \le \frac{1-\frac{r^2}{\kappa}+(2r+(\frac{r^2}{\kappa}+1))}{2} = r+1.
\end{equation}
Now, applying the law of total probability as well as Lemma \ref{Lemma: Growth non-convex concave minimax} shows that
\begin{align*}
    \begin{bmatrix}
    \bbE[\delta_{w,t+1}] \\
    \bbE[\delta_{\theta,t+1}]
    \end{bmatrix} &\le (1-\frac{1}{n})B_t \begin{bmatrix}
    \bbE[\delta_{w,t}] \\
    \bbE[\delta_{\theta,t}] 
    \end{bmatrix}+ \frac{1}{n}\bigl(B_t\begin{bmatrix}
    \bbE[\delta_{w,t}] \\
    \bbE[\delta_{\theta,t}] 
    \end{bmatrix} +2\begin{bmatrix}
    \alpha_{w,t}L_w \\
    \alpha_{\theta,t}L_{\theta} 
    \end{bmatrix} \bigr) \\
    &= B_t \begin{bmatrix}
    \bbE[\delta_{w,t}] \\
    \bbE[\delta_{\theta,t}] 
    \end{bmatrix}+ \begin{bmatrix}
    \frac{2cL_w}{nt} \\
    \frac{2cr^2L_{\theta}}{nt} 
    \end{bmatrix}.\numberthis
\end{align*}
Therefore, over $T$ iterations we will have
\begin{align*}
    \begin{bmatrix}
    \bbE[\delta_{w,T}] \\
    \bbE[\delta_{\theta,T}]
    \end{bmatrix} &\le \sum_{t=t_0+1}^T \bigl\{\prod_{k=t+1}^T B_k \bigr\} \begin{bmatrix}
    \frac{2cL_w}{nt} \\
    \frac{2cr^2L_{\theta}}{nt} 
    \end{bmatrix}  \\
    &= \sum_{t=t_0+1}^T \boldsymbol{\nu}^{-1}\bigl\{\prod_{k=t+1}^T \begin{bmatrix}
    1+\frac{c\ell\lambda_1}{k}  & 0 \\
    0 & 1+\frac{c\ell\lambda_2}{k}
    \end{bmatrix} \bigr\} \boldsymbol{\nu} \begin{bmatrix}
    \frac{2cL_w}{nt} \\
    \frac{2cr^2L_{\theta}}{nt} 
    \end{bmatrix}.\numberthis
\end{align*}
Hence, denoting the minimum and maximum singular values of $\boldsymbol{\nu}$ with  $\sigma_{\min}(\boldsymbol{\nu}),\sigma_{\max}(\boldsymbol{\nu})$ and noting that ${\boldsymbol{\nu}}^{-1}$'s operator norm is equal to $1/\sigma_{\min}(\boldsymbol{\nu})$ we will have 
\begin{align*}
    \bigg\Vert \begin{bmatrix}
    \bbE[\delta_{w,T}] \\
    \bbE[\delta_{\theta,T}]
    \end{bmatrix}\bigg\Vert_2 
    &\le  \frac{\sigma_{\max}(\boldsymbol{\nu})}{\sigma_{\min}(\boldsymbol{\nu})}\sum_{t=t_0+1}^T \bigg\Vert \bigl\{\prod_{k=t+1}^T \begin{bmatrix}
    1+\frac{c\ell\lambda_1}{k}  & 0 \\
    0 & 1+\frac{c\ell\lambda_2}{k}
    \end{bmatrix} \bigr\}  \begin{bmatrix}
    \frac{2cL_w}{nt} \\
    \frac{2cr^2L_{\theta}}{nt} 
    \end{bmatrix}\biggr\Vert_2 \\
    &\le  \frac{\sigma_{\max}(\boldsymbol{\nu})}{\sigma_{\min}(\boldsymbol{\nu})}\sum_{t=t_0+1}^T \biggl\Vert\prod_{k=t+1}^T \begin{bmatrix}
    \exp(\frac{c\ell\lambda_1}{k})  & 0 \\
    0 & \exp(\frac{c\ell\lambda_2}{k})
    \end{bmatrix} \biggr\Vert_2  \big\Vert\begin{bmatrix}
    \frac{2cL_w}{nt} \\
    \frac{2cr^2L_{\theta}}{nt} 
    \end{bmatrix}\big\Vert_2 \\ 
    &=  \frac{\sigma_{\max}(\boldsymbol{\nu})}{\sigma_{\min}(\boldsymbol{\nu})}\sum_{t=t_0+1}^T  \biggl\Vert\begin{bmatrix}
    \exp(\sum_{k=t+1}^T\frac{c\ell\lambda_1}{k})  & 0 \\
    0 & \exp(\sum_{k=t+1}^T\frac{c\ell\lambda_2}{k})
    \end{bmatrix}\biggr\Vert_2 \bigl\Vert\begin{bmatrix}
    \frac{2cL_w}{nt} \\
    \frac{2cr^2L_{\theta}}{nt} 
    \end{bmatrix} \bigr\Vert_2\\ 
    &\le  \frac{\sigma_{\max}(\boldsymbol{\nu})}{\sigma_{\min}(\boldsymbol{\nu})}\sum_{t=t_0+1}^T  \exp(\sum_{k=t+1}^T\frac{c\ell(r+1)}{k})\big\Vert\begin{bmatrix}
    \frac{2cL_w}{nt}   \\
    \frac{2cr^2L_{\theta}}{nt} 
    \end{bmatrix}\big\Vert_2 \\
    &\le  \frac{2crL\sigma_{\max}(\boldsymbol{\nu})}{n\sigma_{\min}(\boldsymbol{\nu})}\sum_{t=t_0+1}^T  \frac{\exp(\sum_{k=t+1}^T\frac{c\ell(r+1)}{k})}{t} \\
    &=  \frac{2crL\sigma_{\max}(\boldsymbol{\nu})T^{c\ell(r+1)}}{n\sigma_{\min}(\boldsymbol{\nu})}\sum_{t=t_0+1}^T  t^{-c\ell(r+1)-1} \\
    & \le  \frac{2rL\sigma_{\max}(\boldsymbol{\nu})}{(r+1)\ell n\sigma_{\min}(\boldsymbol{\nu})} \bigl(\frac{T}{t_0}\bigr)^{c\ell(2r+1)} \\
    & \le  \frac{12L}{n\ell } \bigl(\frac{T}{t_0}\bigr)^{c\ell(r+1)} .\numberthis
\end{align*}
We note that assuming $r\ge 1$ we have $\boldsymbol{\nu}$'s condition number ${\sigma_{\max}(\boldsymbol{\nu})}/{\sigma_{\min}(\boldsymbol{\nu})}\le (\sqrt{2}+1)/(\sqrt{2}-1)\le 6$. This is because given an eigenvalue $\lambda$ of $\begin{bmatrix}
    1  & 1 \\
    r^2 & -r^2/\kappa
    \end{bmatrix}$ and its corresponding eigenvector $[\nu_1,\nu_2]$ we have $\nu_2 = (\lambda-1)\nu_1$ and hence the eigenvector aligns with $[1,\lambda-1]$. Therefore, we can bound the condition number of the following symmetric matrix, because we can consider any vector column along the eigenvector's direction:  
    $$\begin{bmatrix}
    1  & \lambda_1-1 \\
    \lambda_1-1 & (\lambda_1-1)(\lambda_2-1)
    \end{bmatrix}= \begin{bmatrix}
    1  & \frac{-1-\frac{r^2}{\kappa}+\sqrt{4r^2+(1+r^2/\kappa)^2}}{2} \\
    \frac{-1-\frac{r^2}{\kappa}+\sqrt{4r^2+(1+r^2/\kappa)^2}}{2} & -r
    \end{bmatrix}.$$
Since the above matrix is symmetric, its eigenvalues have the same absolute value as its singular values, and therefore the condition number will be bounded as
\begin{align*}
    \frac{\sigma_{\max}(\boldsymbol{\nu})}{\sigma_{\min}(\boldsymbol{\nu})} &\le \frac{\sqrt{(r-1)^2+4(r+(\lambda_1-1)^2)} +(r-1)}{\sqrt{(r-1)^2+4(r+(\lambda_1-1)^2)} -(r-1)} \\
    &\le \frac{\sqrt{(r-1)^2+4(r+(r-\frac{r+1}{2})^2)} +(r-1)}{\sqrt{(r-1)^2+4(r+(r-\frac{r+1}{2})^2)} -(r-1)} \\
    &\le \frac{\sqrt{(r-1)^2+4(r-\frac{r+1}{2})^2} +(r-1)}{\sqrt{(r-1)^2+4(r-\frac{r+1}{2})^2} -(r-1)}  \\
    &= \frac{\sqrt{2(r-1)^2} +(r-1)}{\sqrt{2(r-1)^2} -(r-1)} \\
    &=\frac{\sqrt{2}+1}{\sqrt{2}-1}.\numberthis
\end{align*}
As a result, we showed that conditioned to $\delta_{t_0}=0$ we will have
\begin{equation}
    \bbE\bigl[\delta_{w,T} \big\vert \delta_{t_0}=0 \bigr] \le \frac{12L}{n\ell } \bigl(\frac{T}{t_0}\bigr)^{c\ell(r+1)}.
\end{equation}
Combining the above equation with Lemma \ref{Lemma: non-convex bounded}, we obtain that
\begin{equation}
    \forall \bz,t_0: \;\; \bbE\bigl[\vert f_{\max}(\bw_T;\bz) - f_{\max}(\bw'_T;\bz)
    \vert \bigr] \le \frac{t_0}{n}+ \frac{12LL_w}{n\ell } \bigl(\frac{T}{t_0}\bigr)^{c\ell(r+1)}.
\end{equation}
The above bound will be approaximately minimized at $$t_0=(12(r+1)cLL_w)^{\frac{1}{(r+1)c\ell +1}}T^{\frac{(r+1)c\ell}{(r+1)c\ell +1}}$$ 
which leads to the following bound
\begin{equation}
    \forall \bz: \;\; \bbE\bigl[\vert f_{\max}(\bw_T;\bz) - f_{\max}(\bw'_T;\bz)
    \vert \bigr] \le \frac{1+\frac{1}{(r+1)c\ell }}{n}(12(r+1)cLL_w)^{\frac{1}{(r+1)c\ell +1}}T^{\frac{(r+1)c\ell }{(r+1)c\ell +1}}.
\end{equation}
The theorem's bound on SGDA updates is then a consequence of Theorem 2.2 in \citep{hardt2016train}.
 
For the theorem's bound on SGDmax updates, note that $f_{\max}(\bw;\bz)$ will be $L_w$-Lipschitz. Also, Lemma \ref{Lemma: Smoothness of maximized objective} implies that $f_{\max}(\bw;\bz)$ will be $\ell(\frac{\kappa}{2}+1)$-smooth in $\bw$. Therefore, the result directly follows from Theorem 3.12 in \citep{hardt2016train}. 
\end{proof}

\subsection{Proof of Theorem 6
}
\begin{thm*}
Let minimax cost $0\le f(\cdot,\cdot;\bz)\le 1$ be a bounded non-convex non-concave objective which satisfies Assumptions 1 and 2. Then, the SGDA algorithm with vanishing stepsizes $\max\{\alpha_{w,t},\alpha_{\theta,t}\}\le {c}/{t}$ for constant $c>0$ satisfies the following bound over $T$ iterations: 
    \begin{equation}
        \epsilon_{\gen}(\sgda)\le \frac{1+\frac{1}{\ell  c}}{n}\bigl(2cLL_w\bigr)^{\frac{1}{\ell c+1}}T^{\frac{\ell c}{\ell c+1}}.
    \end{equation}
\end{thm*}
\begin{proof}
To show this result, we apply Lemma \ref{Lemma: non-convex bounded}. Defining $\delta_t=\sqrt{\Vert\bw_t-\bw'_t \Vert^2+\Vert\btheta_t-\btheta'_t \Vert^2}$ for the norm difference of parameters learned by SGDA over two datasets $S,S'$ with one different sample, according to the law of total probability we have:
\begin{align*}
    \mathbb{E}[\delta_{t+1}] &\le (1-\frac{1}{n})(1+\frac{c\ell}{t})\mathbb{E}[\delta_{t}] + \frac{1}{n}\bigl((1+\frac{c\ell}{t})\mathbb{E}[\delta_{t}] + \frac{2cL}{t} \bigr) \\
    &= (1+\frac{c\ell}{t})\mathbb{E}[\delta_{t}] + \frac{2cL}{nt}.\numberthis
\end{align*}
As a result, conditioned on $\delta_{t_0}=0$ we will have
\begin{align*}
    \mathbb{E}[\delta_{T}\big\vert \delta_{t_0}=0] &\le \sum_{t=t_0+1}^T \prod_{k=t+1}^T \bigl\{1+  \frac{c\ell}{k} \bigr\}\frac{2cL}{nt} \\
    &\le \sum_{t=t_0+1}^T \prod_{k=t+1}^T \bigl\{\exp(\frac{c\ell}{k}) \bigr\}\frac{2cL}{nt} \\
    &= \sum_{t=t_0+1}^T  \exp\bigl(\sum_{k=t+1}^T\frac{c\ell}{k}\bigr) \frac{2cL}{nt} \\
    &\le \sum_{t=t_0+1}^T  \exp({c\ell}{\log(T/t)}) \frac{2cL}{nt} \\
    &= \frac{2cLT^{c\ell}}{n}\sum_{t=t_0+1}^T  t^{-c\ell-1} \\
    &\le \frac{2L}{n\ell}\bigl(\frac{T}{t_0}\bigr)^{c\ell}.\numberthis
\end{align*}
Therefore, Lemma \ref{Lemma: non-convex bounded} shows that for every $t_0$ and $\bz$:
\begin{equation}
    \mathbb{E}\bigl[\vert f_{\max}(\bw_t;\bz) -f_{\max}(\bw'_t;\bz) \vert \bigr] \le \frac{t_0}{n}+\frac{2LL_w}{n\ell}\bigl(\frac{T}{t_0}\bigr)^{c\ell}.
\end{equation}
The above upper-bound will be approximately minimized at
\begin{equation}
    t_0 = (2cLL_w)^{\frac{1}{\ell c +1}}T^{\frac{\ell c}{\ell c +1}}.
\end{equation}
Plugging in the above $t_0$ to the upper-bound we obtain the following bound for every $\bz$:
\begin{equation}
    \mathbb{E}\bigl[\vert f_{\max}(\bw_t;\bz) -f_{\max}(\bw'_t;\bz) \vert \bigr] \le \frac{1+\frac{1}{\ell c}}{n}(2cLL_w)^{\frac{1}{\ell c +1}}T^{\frac{\ell c}{\ell c +1}}.
\end{equation}
The above result combined with Theorem 2.2 from \citep{hardt2016train} proves the theorem.
\end{proof}

\end{appendices}


\end{document}